\documentclass{article}

\usepackage[preprint]{neurips_2026}

\usepackage[utf8]{inputenc}
\usepackage[T1]{fontenc}
\usepackage{hyperref}
\usepackage{url}
\usepackage{booktabs}
\usepackage{amsfonts}
\usepackage{amsmath}
\usepackage{amssymb}
\usepackage{amsthm}
\usepackage{nicefrac}
\usepackage{microtype}
\usepackage[table]{xcolor}
\usepackage{graphicx}
\usepackage{algorithm}
\usepackage{algorithmic}
\providecommand{\RETURN}{\STATE \textbf{return} }
\usepackage{enumitem}
\usepackage{multirow}
\usepackage{pifont}

\newtheorem{theorem}{Theorem}
\newtheorem{proposition}[theorem]{Proposition}
\newtheorem{lemma}[theorem]{Lemma}
\newtheorem{corollary}[theorem]{Corollary}

\newtheorem{assumption}{Assumption}
\newtheorem{remark}{Remark}

\newcommand{\cmark}{{\color{green!60!black}\ding{108}}}
\newcommand{\hmark}{{\color{yellow!70!orange}\ding{119}}}
\newcommand{\xmark}{{\color{gray!40}\ding{108}}}

\newcommand{\Xcal}{\mathcal{X}}
\newcommand{\Ycal}{\mathcal{Y}}
\newcommand{\Dcal}{\mathcal{D}}
\newcommand{\Fcal}{\mathcal{F}}
\newcommand{\Gcal}{\mathcal{G}}
\newcommand{\Ucal}{\mathcal{U}}
\newcommand{\Ncal}{\mathcal{N}}
\newcommand{\Kcal}{\mathcal{K}}
\newcommand{\Psocial}{P_{\mathrm{social}}}
\newcommand{\Yhat}{\hat{Y}}

\newcommand{\epsact}{\varepsilon_{\mathrm{act}}}
\newcommand{\Lbar}{\bar{L}}
\newcommand{\omegahat}{\hat{\omega}}
\newcommand{\qtilde}{\tilde{q}}
\newcommand{\btk}{b_{t,k}}

\title{Budgeted Act-or-Defer Multi-Agent LLM Deliberation\\with Local Reliability Bounds}

\author{%
 \textbf{Mengdie Flora Wang \textsuperscript{1}},
 \textbf{Haochen Xie \textsuperscript{1}},
\textbf{Guanghui Wang \textsuperscript{1}},
 \textbf{Devin Zhang \textsuperscript{2}},
 \textbf{Jae Oh Woo\textsuperscript{1}}
 \\
 \textsuperscript{1}AWS Generative AI Innovation Center, USA \\
 \textsuperscript{2} General Motors (GM), USA
}

\begin{document}

\maketitle

\begin{abstract}
Multi-agent deliberation among LLMs can improve reasoning, but deployment requires deciding when the current answer is reliable enough to act on and when it should be escalated to human review. We formulate this as budgeted act-or-defer decision making. At each round, the system maps the debate prefix to a low-dimensional state, computes a $k$-nearest-neighbor lower confidence bound on state-conditional correctness using calibration data, and acts only when the bound exceeds a user-specified reliability threshold. The certificate controls wrong actions through the decomposition $\beta = \delta + \alpha + \epsact$, separating calibration failure, residual action risk, and representation gap. The guarantee is conditional, not distribution-free: it relies on a valid local bias envelope and an action-region representation-gap bound, and each assumption is paired with falsification-style diagnostics. Because the same absolute wrong-action budget has different meanings across tasks of different difficulty, we set budgets relative to each task's final-round error using training data only, and evaluate safety by normalized budget usage $\mathrm{WA}/\beta$. On six benchmarks against nine baselines, the method uses 9--12\% of the pre-declared budget on activated datasets, reaching up to 84\% automation and 96\% acted-on accuracy; on stress-test datasets, it defers rather than forcing unreliable automation. Rather than relying on per-task post-hoc threshold search, the method prospectively converts a user-declared wrong-action budget into an auditable act-or-defer operating point before deployment, under explicitly stated assumptions.
\end{abstract}

\section{Introduction}
\label{sec:intro}

Multi-agent debate, where multiple LLM instances exchange responses over several rounds, has been shown to improve reasoning \citep{du2023improving, liang2024encouraging}. To deploy such systems in practice, one needs a stopping criterion with principled control over the probability of acting on a wrong answer. In human-in-the-loop deployments this stopping decision is also an escalation decision: the system should act autonomously only when the local reliability bound supports action, and otherwise defer the instance to human review. Existing approaches either operate at a fixed final round with only marginal coverage (e.g., split conformal prediction \citep{vovk2005algorithmic, angelopoulos2023conformal}) or rely on heuristic consensus thresholds without explicit wrong-action budgets. Neither provides per-instance, per-round reliability bounds that enable adaptive stopping.

We address this gap by constructing local lower confidence bounds on the correctness probability using $k$-nearest-neighbor ($k$-NN) neighborhoods in a debate-state metric space. The debate prefix at round $t$ is mapped to a state $U_t$ in a metric space $(\Ucal, d)$. The local correctness function $q_t(u) = \Pr(Y = \Yhat_t \mid U_t = u)$ is bounded from below via a \emph{local bias envelope} $\btk(u)$ that controls the $k$-NN approximation error. For each candidate neighborhood size $k$ in a pre-specified family $\Kcal$, we subtract the bias envelope and a Hoeffding concentration term from the $k$-NN empirical mean, then take the maximum over $\Kcal$:
\begin{equation}
\label{eq:intro_bound}
    L_t = \max_{k \in \Kcal}\left[\widehat{q}_{t,k}(U_t) - \btk(U_t) - \sqrt{\frac{\log(T|\Kcal|/\delta)}{2k}}\right],
\end{equation}
where the $\log|\Kcal|$ penalty absorbs the adaptivity in $k$. The bias envelope $\btk(u)$ is a generic primitive (Lipschitz, H\"{o}lder, bounded-oscillation models are special cases). The first correction controls approximation bias, the second controls sampling variance, and the maximization over $\Kcal$ selects the tightest local scale per instance.

A central design choice is \emph{budget alignment}. The user specifies a risk budget $\beta$ and the method decomposes it as
\begin{equation}
\label{eq:budget}
    \beta \;=\; \underbrace{\delta}_{\substack{\text{statistical}\\\text{failure}}} \;+\; \underbrace{\alpha}_{\substack{\text{residual}\\\text{action risk}}} \;+\; \underbrace{\epsact}_{\substack{\text{representation}\\\text{gap}}}.
\end{equation}
The residual $\alpha = \beta - \delta - \epsact$ directly sets the stopping threshold $1 - \alpha$. The stopping rule $\tau = \inf\{t : L_t \geq 1 - \alpha\}$ then satisfies $\Pr(\text{act and wrong}) \leq \beta$ conditional on local assumptions (Theorem~\ref{thm:wrong_action}). The contribution is not the mathematical tools (Hoeffding + union bound, which are classical) but a risk-decomposition framework for adaptive deliberation stopping under auditable local assumptions.

\begin{figure}[t]
\centering
\includegraphics[width=\textwidth]{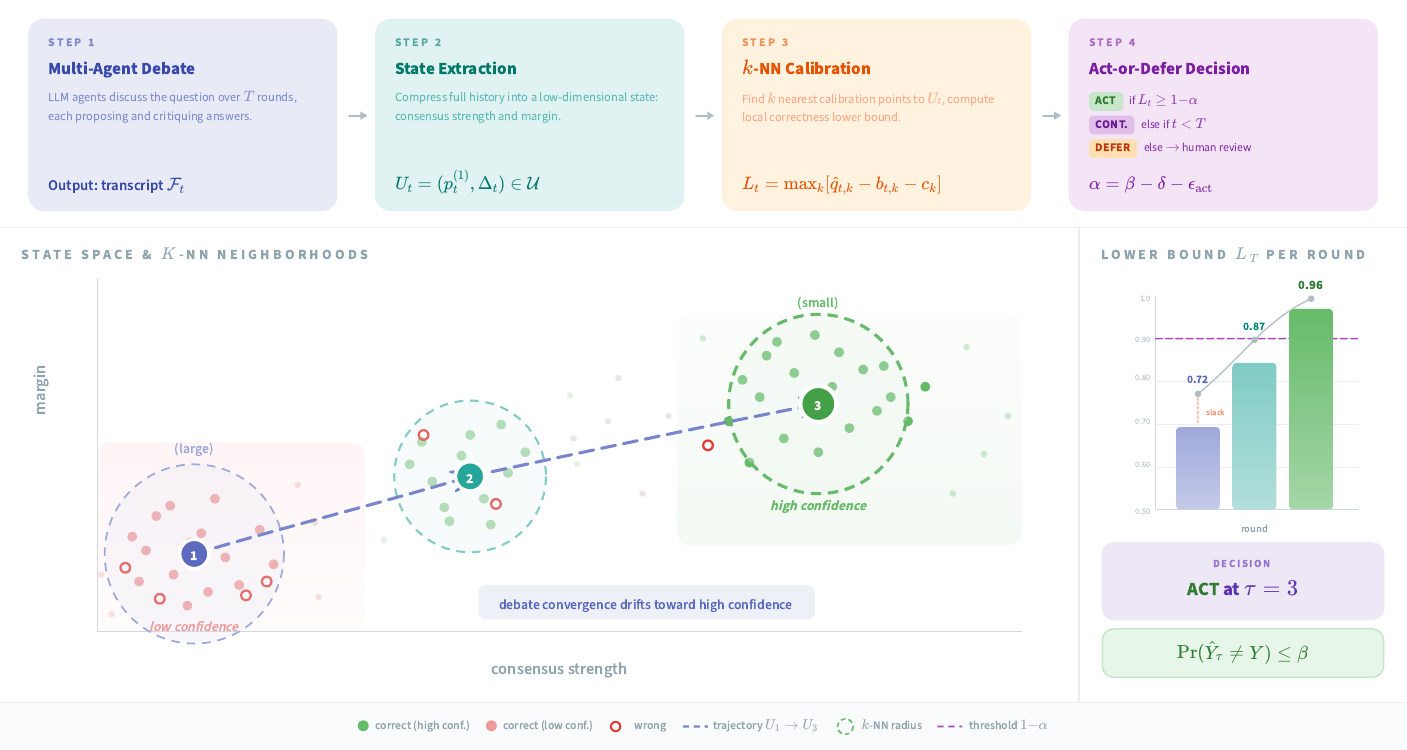}
\caption{\textbf{Overview.}
\textit{Top}: (1)~Agents debate over $T$ rounds; (2)~each transcript is compressed to state $U_t = \phi(\mathcal{F}_t)$; (3)~$k$-NN lookup yields lower bound $L_t$; (4)~act if $L_t \geq 1-\alpha$, continue if $t<T$, else defer to human review.
\textit{Bottom left}: calibration states colored by correctness; the test state drifts from low- to high-confidence as debate progresses.
\textit{Bottom right}: $L_t$ grows from 0.72 to 0.96, crossing $1-\alpha = 0.90$ at $\tau = 3$.}
\label{fig:pipeline}
\end{figure}

The guarantee is conditional on auditable local assumptions, not distribution-free. Our contributions:
\textbf{(i)} We formalize adaptive stopping for multi-agent deliberation as a local state-conditional risk-control problem with budget decomposition $\beta = \delta + \alpha + \epsact$.
\textbf{(ii)} We introduce a local bias envelope $\btk(u)$ as a plug-in interface and prove simultaneous validity $L_t \leq q_t(U_t)$ for all $(t,k)$ with probability $\geq 1 - \delta$, plus a wrong-action guarantee (Theorems~\ref{thm:pointwise}--\ref{thm:wrong_action}) and an oracle slack inequality (Theorem~\ref{thm:oracle_slack}, Appendix~\ref{app:proofs}). Each assumption is paired with a diagnostic proxy.
\textbf{(iii)} Lipschitz, H\"{o}lder, and bounded-oscillation envelopes are recovered as corollaries. Our default empirical modulus absorbs estimation failure into $\delta$ (Corollary~\ref{cor:modulus}) and converges at $O(n^{-1/(2+D)})$.
\textbf{(iv)} We introduce a relative-risk deployment policy: $\beta_d = \lambda^\star_d \widehat{e}_{T,d}^{\,\mathrm{train}}$ is induced from training data alone, making budgets comparable across tasks and turning normalized budget usage $\mathrm{WA}/\beta$ into the primary safety metric.
\textbf{(v)} Experiments on six benchmarks ($n = 546$--$8{,}678$) against nine baselines show that, on activated datasets, the method uses $\leq$12\% of the pre-declared budget at up to 97\% acted-on accuracy. Matched-automation diagnostics show that tuned scalar thresholds can achieve similar subset quality after per-task search, but the proposed certificate constructs the operating point prospectively from a declared wrong-action budget and makes the required assumptions explicit.

\section{Problem Setup}
\label{sec:setup}

Consider a multiple-choice task with input space $\Xcal$ and label space $\Ycal = \{1, \ldots, C\}$. An ensemble of $N$ LLM agents debates for at most $T$ rounds. At round $t$, the debate has produced a filtration $\Fcal_t$ containing all agent responses and per-agent belief distributions $\{P^{(t)}_j(y \mid x)\}_{j=1}^{N}$ up to round $t$. We aggregate these into a \emph{social belief}
\begin{equation}
\label{eq:psocial}
    \Psocial^{(t)}(y \mid x) \;=\; \frac{1}{N}\sum_{j=1}^{N} P^{(t)}_j(y \mid x),
\end{equation}
and take $\Yhat_t = \arg\max_y \Psocial^{(t)}(y \mid x)$ as the round-$t$ prediction (Appendix~\ref{app:details} specifies how each $P^{(t)}_j$ is elicited).

\paragraph{State embedding.}
Each debate prefix is mapped to a state $U_t = \phi_t(\Fcal_t)$ in a metric space $(\Ucal, d)$ via a deterministic, $\Fcal_t$-measurable $\phi_t$. Typical components include the top social probability $p_t^{(1)} = \max_y \Psocial^{(t)}(y \mid x)$, the margin $\Delta_t = p_t^{(1)} - p_t^{(2)}$, inter-agent disagreement, and top-answer stability; the metric $d$ reflects correctness relevance (e.g., Euclidean on a normalized feature vector).

\paragraph{Geometry--sufficiency trade-off.}
The choice of $\phi_t$ jointly controls three quantities central to the framework:
\textbf{(a)}~\emph{smoothness of $q_t$}, which governs the bias envelope $\btk$;
\textbf{(b)}~\emph{representation gap $\epsact$}, the information lost in compressing $\Fcal_t$ to $U_t$;
\textbf{(c)}~\emph{intrinsic dimension $D$}, which enters the certification rate $O(n^{-1/(2+D)})$.
Richer states reduce $\epsact$ but raise $D$ and may roughen $q_t$; minimal states keep $D$ small and $q_t$ smooth but risk larger $\epsact$. We investigate this trade-off empirically in Section~\ref{sec:exp_phi}.

\paragraph{Local correctness and goal.}
The target is the \emph{local correctness function}
\begin{equation}
\label{eq:local_target}
    q_t(u) := \Pr(Y = \Yhat_t \mid U_t = u),
\end{equation}
a pointwise quantity conditioning on the exact state rather than averaging over a region. We seek a stopping rule $\tau = \inf\{t \leq T : L_t \geq 1-\alpha\}$ with a data-driven lower bound $L_t$ for $q_t(U_t)$, so that $\Pr(\tau \leq T,\; \Yhat_\tau \neq Y) \leq \beta$. The budget decomposition~\eqref{eq:budget} controls each term separately: $\delta$ by Hoeffding with a $(t,k)$ union bound, $\alpha$ by the stopping threshold, and $\epsact$ by the action-region representation gap.

\section{Method}
\label{sec:method}

The method turns $\beta$ into a stopping rule in two steps: fix the action threshold $1-\alpha$ with $\alpha = \beta - \delta - \epsact$ from the budget decomposition~\eqref{eq:budget}, then build a lower bound $L_t$ certifying when the threshold is met. The main guarantee (Theorem~\ref{thm:wrong_action}) follows by summing worst cases per component.

\paragraph{Offline: calibration.}
On $\Dcal_{\mathrm{cal}} = \{(X_i, Y_i)\}_{i=1}^n$, we run the full $T$-round debate once per instance and record, for every $(t,i)$, the state $U_{i,t} = \phi_t(\Fcal_{i,t})$ and correctness indicator $Z_{i,t} = \mathbf{1}[Y_i = \Yhat_{i,t}]$. We split $\Dcal_{\mathrm{cal}}$ into $\Dcal_{\mathrm{mod}}$ (20\%, for bias-envelope estimation) and $\Dcal_{\mathrm{loc}}$ (80\%, the database searched at test time).

\paragraph{Online: $k$-NN lookup.}
At test round $t$ with state $U_t$, and for each $k$ in a pre-specified family $\Kcal = \{k_1, \ldots, k_m\}$, let $\Ncal_{t,k}(U_t) \subset \Dcal_{\mathrm{loc}}$ denote the $k$ nearest calibration instances and $h_{t,k}(U_t) = \max_{i \in \Ncal_{t,k}(U_t)} d(U_{i,t}, U_t)$ the $k$-NN radius. The raw empirical local correctness is $\widehat{q}_{t,k}(U_t) = \frac{1}{k} \sum_{i \in \Ncal_{t,k}(U_t)} Z_{i,t}$.

\paragraph{Adaptive lower bound.}
The raw estimate $\widehat{q}_{t,k}$ is optimistic on two counts: neighbors are \emph{near} $U_t$ but not \emph{at} $U_t$ (local bias), and with only $k$ samples the empirical mean may exceed the truth (sampling variance). We subtract a correction for each and take the tightest bound across $k$:
\begin{align}
    L_{t,k}(U_t) &\;=\; \widehat{q}_{t,k}(U_t) \;-\; \underbrace{\btk(U_t)}_{\text{bias}} \;-\; \underbrace{\sqrt{\tfrac{\log(T|\Kcal|/\delta)}{2k}}}_{\text{variance}}, \label{eq:local_lower}\\
    L_t &\;=\; \max_{k \in \Kcal} L_{t,k}(U_t). \label{eq:adaptive_lower}
\end{align}
Each test instance automatically picks the $k$ with the best bias--variance trade-off at its location; the $\log|\Kcal|$ term pays for adaptivity via a union bound. By construction $L_{t,k} \leq q_t(U_t)$ with high probability (Theorem~\ref{thm:pointwise}).

\paragraph{Stopping rule.}
The system acts at $\tau = \inf\{t \leq T : L_t \geq 1 - \alpha\}$ and defers otherwise; pseudocode is in Algorithm~\ref{alg:main} (Appendix~\ref{app:details}).

\paragraph{Bias envelope instantiations.}
The bias correction $\btk(u)$ is a plug-in interface: any conservative bound on how much $q_t$ can vary within the $k$-NN neighborhood is valid.
\begin{itemize}[nosep,leftmargin=1.5em]
    \item \emph{Lipschitz}: $\btk(u) = L_t^\star \cdot h_{t,k}(u)$ --- global slope bound.
    \item \emph{H\"{o}lder}: $\btk(u) = C_t \cdot h_{t,k}(u)^s$ --- sub-linear growth for smoother functions.
    \item \emph{Empirical modulus} (default): fit a smoothed pilot $\qtilde_t$ on $\Dcal_{\mathrm{mod}}$ (logistic GAM; Appendix~\ref{app:smoothers}), then measure how much $\qtilde_t$ varies within radius $r$:
\end{itemize}
\begin{equation}
\label{eq:modulus}
    \btk(u) = \omegahat_t(h_{t,k}(u)), \quad \omegahat_t(r) = \max_{\substack{i,j \in \Dcal_{\mathrm{mod}} \\ d(U_{i,t}, U_{j,t}) \leq r}} |\qtilde_t(U_{i,t}) - \qtilde_t(U_{j,t})|.
\end{equation}
The empirical modulus avoids the pathology of estimating a global Lipschitz constant from binary labels (where pairwise slopes diverge at short distances). Its estimation failure is absorbed into $\delta$ (Corollary~\ref{cor:modulus}), keeping the three-way budget intact. We stress-test validity by $\rho$-inflation ($\btk \to \rho\,\btk$, $\rho \in \{1, 1.5, 2\}$); Proposition~\ref{prop:assumption_cost} gives $\Pr(\text{wrong act}) \leq \beta + \Pr(\mathcal{E}_{\mathrm{env}}^c(\rho))$, and Table~\ref{tab:audit} shows $2\times$ inflation reduces Act by only 1--3\,pp.

\section{Assumptions}
\label{sec:assumptions}

The guarantees in Section~\ref{sec:theory} rest on four assumptions. The first two are standard; the last two are specific to the local approach and constitute the price for pointwise guarantees. Each is paired with a diagnostic (Table~\ref{tab:audit}).

\begin{assumption}[Across-instance exchangeability]
\label{assum:exch}
The calibration examples $\{(X_i, Y_i)\}_{i=1}^n$ are i.i.d.\ (or exchangeable) and independent of the test example $(X, Y)$. Within each instance, debate rounds may be arbitrarily dependent.
\end{assumption}

\begin{assumption}[State measurability]
\label{assum:measurable}
For each round $t$, the prediction $\Yhat_t$ and state $U_t = \phi_t(\Fcal_t)$ are $\Fcal_t$-measurable.
\end{assumption}

\begin{assumption}[Local bias envelope --- plug-in interface]
\label{assum:bias}
For each round $t$ and candidate $k \in \Kcal$, there exists a measurable function $\btk : \Ucal \to [0, \infty)$ such that
$\big|\frac{1}{k}\sum_{i \in \Ncal_{t,k}(u)} q_t(U_{i,t}) - q_t(u)\big| \leq \btk(u)$ a.s.\ for all $u \in \Ucal$, where $\Ncal_{t,k}(u)$ is the $k$-NN of $u$ among calibration states.
\label{eq:bias_envelope}
\end{assumption}

This is an \emph{interface}, not a specific smoothness model: any user-supplied conservative envelope can be plugged into~\eqref{eq:local_lower}, and the theorems hold as long as the envelope is valid. The estimation error decomposes as sampling fluctuation (controlled by Hoeffding) plus local bias (controlled by $\btk$):
\begin{equation}
\label{eq:bias_variance_decomp}
    \widehat{q}_{t,k}(u) - q_t(u) = \underbrace{\widehat{q}_{t,k}(u) - \tfrac{1}{k}\textstyle\sum_{i \in \Ncal_{t,k}(u)} q_t(U_{i,t})}_{\text{sampling fluctuation}} + \underbrace{\tfrac{1}{k}\textstyle\sum_{i \in \Ncal_{t,k}(u)} q_t(U_{i,t}) - q_t(u)}_{\text{local bias}}.
\end{equation}
Lipschitz, H\"{o}lder, and bounded-oscillation models are all special cases (Section~\ref{sec:method}). The guarantee is only as good as the envelope: if $\btk$ underestimates the true bias, validity degrades by exactly $\Pr(\mathcal{E}_{\mathrm{env}}^c)$ (Proposition~\ref{prop:assumption_cost}).

\begin{assumption}[Action-region representation gap (one-sided)]
\label{assum:epsact}
There exists $\epsact \geq 0$ such that for every round $t$ and every realized debate history $f_t$ with $u = \phi_t(f_t)$ satisfying $q_t(u) \geq 1 - \alpha$:
\begin{equation}
\label{eq:epsact}
    \big[q_t(u) - \Pr(Y = \Yhat_t \mid \Fcal_t = f_t)\big]^{+} \;\leq\; \epsact \quad \text{a.s.}
\end{equation}
where $[x]^{+} = \max(x,0)$.
\end{assumption}

This is weaker than a global two-sided gap: it is \emph{one-sided} (only overestimation matters) and \emph{restricted to the action region} $\{(t,u) : q_t(u) \geq 1 - \alpha\}$. Lipschitz smoothness and local support are \emph{not} required by the main theorems---they are optional instantiations of $\btk$, formalized in Appendix~\ref{app:proofs}.

\section{Theoretical Guarantees}
\label{sec:theory}

All proofs are in Appendix~\ref{app:proofs}. We build from pointwise validity to the main wrong-action guarantee. Define the bad event $E_{t,k} := \{L_{t,k}(U_t) > q_t(U_t)\}$.

\begin{theorem}[Pointwise validity]
\label{thm:pointwise}
Under Assumptions~\ref{assum:exch}--\ref{assum:bias}, for any fixed $t$, $k \in \Kcal$, and $\delta_{t,k} \in (0,1)$:
$\Pr(E_{t,k}) \leq \delta_{t,k}$.
\label{eq:bad_event}
\end{theorem}

\begin{theorem}[Simultaneous validity]
\label{thm:simultaneous}
Under Assumptions~\ref{assum:exch}--\ref{assum:bias}, set $\delta_{t,k} = \delta/(T|\Kcal|)$. Then $\Pr(\bigcup_{t,k} E_{t,k}) \leq \delta$; equivalently, $L_t = \max_k L_{t,k}(U_t) \leq q_t(U_t)$ for all $t$ with probability $\geq 1 - \delta$.
\end{theorem}

\begin{theorem}[Budget-aligned wrong-action control]
\label{thm:wrong_action}
Under Assumptions~\ref{assum:exch}--\ref{assum:epsact}, the stopping rule $\tau = \inf\{t \leq T : L_t(U_t) \geq 1-\alpha\}$ satisfies
\begin{equation}
\label{eq:wrong_action_bound}
    \Pr\big(\tau \leq T,\; \Yhat_\tau \neq Y\big) \leq \delta + \alpha + \epsact.
\end{equation}
In particular, if $\delta + \alpha + \epsact \leq \beta$, then $\Pr(\tau \leq T,\; \Yhat_\tau \neq Y) \leq \beta$.
\end{theorem}

\noindent\textit{Scope.} Theorem~\ref{thm:wrong_action} does not provide distribution-free conditional correctness, nor does it verify the state representation automatically. It states that, given a valid local bias envelope and an action-region representation-gap bound, the stopping rule converts these quantities into a wrong-action guarantee. Assumption violations add explicitly to the risk (Proposition~\ref{prop:assumption_cost} below).

Each term controls a distinct error source (Figure~\ref{fig:budget}): $\delta$ covers calibration failure, $\alpha$ is the residual wrong-answer probability when $q_t \geq 1-\alpha$, and $\epsact$ captures representation loss. We make the cost of assumption failure explicit:

\begin{proposition}[Cost of assumption failure]
\label{prop:assumption_cost}
Let $\mathcal{E}_{\mathrm{env}}$ be the event that Assumption~\ref{assum:bias} holds in the action region, and $\mathcal{E}_{\mathrm{rep}}$ the event that Assumption~\ref{assum:epsact} holds. Then:
$\Pr(\tau \leq T,\; \Yhat_\tau \neq Y) \leq \delta + \alpha + \epsact + \Pr(\mathcal{E}_{\mathrm{env}}^c) + \Pr(\mathcal{E}_{\mathrm{rep}}^c)$.
\end{proposition}
\noindent The guarantee degrades gracefully: each assumption violation adds its failure probability to the bound.

\begin{corollary}[Acted-on error rate]
\label{cor:acted_error}
Under the same conditions as Theorem~\ref{thm:wrong_action}, if $p_{\mathrm{act}} := \Pr(\tau \leq T) > 0$:
\begin{equation}
\label{eq:acted_error}
    \Pr\!\big(\Yhat_\tau \neq Y \mid \tau \leq T\big) \;\leq\; \alpha + \epsact + \frac{\delta}{p_{\mathrm{act}}}.
\end{equation}
\end{corollary}
\noindent This connects the unconditional wrong-action bound (Theorem~\ref{thm:wrong_action}) to Acc$|$Act $= 1 - \Pr(\Yhat_\tau \neq Y \mid \tau \leq T)$ reported in experiments. When Act is large (e.g., $p_{\mathrm{act}} = 0.77$ on BBH), the $\delta/p_{\mathrm{act}}$ overhead is small ($0.03/0.77 \approx 0.04$).

Figure~\ref{fig:budget} in Appendix~\ref{app:details} illustrates the budget decomposition and its empirical envelope.

\paragraph{Auxiliary results (deferred).}
Appendix~\ref{app:proofs} shows that adaptive selection over $k\in\Kcal$ satisfies an oracle-style slack inequality (Theorem~\ref{thm:oracle_slack}), so the selected neighborhood is within a constant factor of the best bias--variance trade-off in $\Kcal$. The plug-in envelope $\btk$ recovers Lipschitz, H\"{o}lder, and estimated-modulus constructions as corollaries (Corollaries~\ref{cor:lipschitz_det}--\ref{cor:holder}); formal statements and proofs are given there. Under standard local support and smoothness conditions, Appendix~\ref{app:proofs} derives the nonparametric slack rate $q_t - L_t = O(n^{-1/(2+D)})$ (Corollary~\ref{cor:rate}), which motivates keeping the default state low-dimensional. Group-conditional control (Theorem~\ref{thm:group}) extends the guarantee per group by splitting $\delta$ across $J$ groups.

\paragraph{Claim hierarchy.}
\textbf{(a)}~Formal theorems (Thms.~\ref{thm:pointwise}--\ref{thm:wrong_action}), valid given Assumptions~\ref{assum:exch}--\ref{assum:epsact}; \textbf{(b)}~conditional guarantee~\eqref{eq:wrong_action_bound}, with violation cost made explicit (Prop.~\ref{prop:assumption_cost}); \textbf{(c)}~diagnostic proxies (Table~\ref{tab:audit}), necessary but not sufficient; \textbf{(d)}~empirical observations (Table~\ref{tab:main}), consistent with but not proving the guarantee.

Figure~\ref{fig:bias_variance} in Appendix~\ref{app:details} visualizes the bias--variance trade-off governing the choice of $k$ and the resulting certification-slack rate.

\section{Experiments}
\label{sec:experiments}

\paragraph{Benchmarks.}
We evaluate on six benchmarks spanning a wide range of difficulty and sample size:
\textbf{(1)}~\emph{MMLU-Pro} \citep{wang2024mmlu} (10-option, 8 domains pooled, $n\!=\!8{,}312$);
\textbf{(2)}~\emph{LogiQA} \citep{liu2020logiqa} (4-option logical reasoning, $n\!=\!8{,}678$; base acc.\ ${\sim}76\%$);
\textbf{(3)}~\emph{ARC-Challenge} \citep{clark2018think} (4-option, $n\!=\!2{,}590$; base acc.\ ${\sim}96\%$);
\textbf{(4)}~\emph{BIG-Bench Hard (BBH)} \citep{suzgun2023challenging} (10 subtasks, $n\!=\!2{,}395$; base acc.\ $71$--$100\%$);
\textbf{(5)}~\emph{MuSR} \citep{sprague2024musr} (3 subtasks, $n\!=\!756$; base acc.\ ${\sim}66\%$);
\textbf{(6)}~\emph{GPQA} \citep{rein2023gpqa} (4-option graduate science, $n\!=\!546$; base acc.\ ${\sim}72\%$).

\paragraph{Protocol.}
Three heterogeneous agents (Claude Haiku 4.5, DeepSeek-R1, Qwen-3 32B) debate for $T\!=\!4$ rounds. Data is pooled within each benchmark and split 50/50 into training/test; the training half is further split 20/80 into $\Dcal_{\mathrm{mod}}$/$\Dcal_{\mathrm{loc}}$. Default state: $U_t = (p_t^{(1)}, \Delta_t) \in \mathbb{R}^2$; bias envelope via empirical modulus~\eqref{eq:modulus}; $\Kcal = \{128, 256, 512\}$.

\paragraph{Baselines.}
We compare against nine baselines spanning three families.

\noindent\emph{Risk-control methods} (calibrated to control aggregate WA$\,\leq\beta$ via statistical calibration on the training split):
\textbf{(R1)}~\emph{CRC}~\citep{angelopoulos2024conformal} calibrates a global nonconformity-score quantile so expected risk is bounded;
\textbf{(R2)}~\emph{Selective Prediction}~\citep{elyaniv2010foundations, geifman2017selective} tunes a threshold on the top social probability $p_t^{(1)}$ to control WA$\,\leq\beta$;
\textbf{(R3)}~\emph{Calibrated Learned} trains a HistGradientBoosting classifier on five deliberation features, applies isotonic calibration, then tunes a threshold for WA$\,\leq\beta$;
\textbf{(R4)}~\emph{Isotonic Confidence} applies isotonic regression to recalibrate $p_t^{(1)}$, then tunes a threshold for WA$\,\leq\beta$.

\noindent\emph{Ablations of our method}:
\textbf{(A1)}~\emph{$k$NN no bias} is our $k$-NN lower bound without the bias envelope $\btk$, isolating the contribution of the modulus correction;
\textbf{(A2)}~\emph{Final-round local} applies our full bound only at $t\!=\!T$, measuring the value of adaptive early stopping.

\noindent\emph{Heuristic baselines}:
\textbf{(H1)}~\emph{Consensus} acts when all agents agree on the same answer;
\textbf{(H2)}~\emph{Confidence Threshold} acts when $p_t^{(1)} \geq 0.90$, without calibration;
\textbf{(H3)}~\emph{Learned Stopper} is logistic regression on the state $U_t$, without risk calibration.

All baseline thresholds are tuned exclusively on the training split; reported metrics are on the held-out test split, averaged over 10 random splits. Throughout, Act denotes automation rate, Acc$|$Act acted-on accuracy, WA the wrong-action rate, and Rnd the average stopping round.

\paragraph{Difficulty-normalized deployment budgets.}
\label{para:relative_budget}
A fixed absolute wrong-action budget is not a task-invariant deployment requirement: it corresponds to very different levels of tolerance depending on the task's final-round error. For ARC ($\widehat{e}_T \approx 5\%$), $\beta = 0.10$ is permissive relative to the final-round error; for MuSR ($\widehat{e}_T \approx 34\%$), the same $\beta$ is stringent. Absolute $\beta$ therefore conflates task difficulty with stopping quality. We use $\lambda$ to \emph{parameterize deployment tolerance} relative to the task's final-round error---the fraction of $\widehat{e}_{T,d}$ the user is willing to allocate as autonomous wrong-action budget---and select the operating multiplier $\lambda^\star_d$ from training data only, inducing
\begin{equation}\label{eq:beta_rule}
\beta_d \;=\; \lambda^\star_d \;\widehat{e}_{T,d}^{\,\mathrm{train}},
\end{equation}
where $\widehat{e}_{T,d}^{\,\mathrm{train}} = \Pr_{\Dcal_{\mathrm{train}}}(\widehat{Y}_T \neq Y)$ is the final-round error on the training split. The operating budget is selected before deployment using only training data: we sub-split the training set, sweep $\lambda \in \{0.25, 0.50, \ldots, 5.00\}$, and pick the $\lambda$ that maximizes automation subject to $\widehat{\mathrm{WA}}/\beta \leq 0.10$ on the held-out training sub-split (Appendix~\ref{app:lambda_selection}). We fix $\delta = 0.03$ and $\epsact = 0.02$, giving action threshold $\alpha = \beta - 0.05$. The formal guarantee (Theorem~\ref{thm:wrong_action}) holds for \emph{any} fixed $\beta$; training-only selection simply chooses the operating point before deployment without compromising test-time validity. No test-set quantity is used to choose $\lambda^\star_d$, $\beta_d$, $k$, or the stopping threshold; the held-out test split is used only for evaluation, averaged over 10 random splits. Because $\beta_d$ is task-dependent by design, the primary cross-dataset safety metric is the normalized budget usage $\mathrm{WA}/\beta$, which measures the fraction of the pre-declared wrong-action budget consumed by each method; methods are always compared under the \emph{same} pre-declared $\beta_d$ within each dataset, while $\beta_d$ itself is induced from training-only final-round error to make tolerance comparable across tasks. Large induced $\beta_d$ (e.g., on GPQA/MuSR) is not a deployable safety setting but a \emph{diagnostic} signal that the available calibration sample and representation are insufficient for moderate-risk local certification. A fixed-$\beta$ robustness check applying the same absolute $\beta$ to all datasets is reported in Appendix~\ref{app:fixed_beta} as a consistency check, not an alternative main framing.

\subsection{Main results}

\begin{table*}[t]
\caption{\textbf{Difficulty-normalized, training-selected operating budgets} ($T\!=\!4$, 10 random 50/50 splits). Per-dataset $\beta_d = \lambda^\star_d \hat{e}_{T,d}^{\,\mathrm{train}}$ is induced from training data via~\eqref{eq:beta_rule} before test evaluation; methods within each dataset share the same $\beta_d$, and WA/$\beta$ (budget usage ratio, not raw WA) is the cross-dataset safety metric. Rows grouped by family: our method and its ablations ($k$NN no bias, Final-round), risk-control (CRC, Selective Prediction, Calibrated Learned, Isotonic Confidence), and heuristics (Consensus, Confidence Threshold, Learned Stopper). Fixed-$\beta$ robustness check in Appendix~\ref{app:fixed_beta}.}
\label{tab:main}
\centering
\resizebox{\textwidth}{!}{%
\def\tblsep{\hspace{3pt}\vrule width 0.4pt\hspace{2pt}\vrule width 0.4pt\hspace{3pt}}%
\renewcommand{\arraystretch}{1.22}
\setlength{\tabcolsep}{2.2pt}
\begin{tabular}{@{} l ccccc ccccc ccccc ccccc @{\tblsep} ccccc ccccc @{}}
\toprule
& \multicolumn{20}{c@{\tblsep}}{\cellcolor{blue!8}\small\textsc{Certifiable regime}} & \multicolumn{10}{c@{}}{\cellcolor{orange!10}\small\textsc{Stress-test / insufficient evidence}} \\
\cmidrule(lr){2-21}\cmidrule(lr){22-31}
& \multicolumn{5}{c}{\textbf{LogiQA} {\scriptsize($n\!=\!8678$)}} & \multicolumn{5}{c}{\textbf{MMLU-Pro} {\scriptsize($n\!=\!8312$)}} & \multicolumn{5}{c}{\textbf{ARC} {\scriptsize($n\!=\!2590$)}} & \multicolumn{5}{c@{\tblsep}}{\textbf{BBH} {\scriptsize($n\!=\!2395$)}} & \multicolumn{5}{c}{\textbf{MuSR} {\scriptsize($n\!=\!756$)}} & \multicolumn{5}{c@{}}{\textbf{GPQA} {\scriptsize($n\!=\!546$)}} \\
& \multicolumn{5}{c}{\scriptsize$\hat{e}_T\!=\!.24$, $\lambda^\star\!=\!1.3$, $\beta\!=\!.31$} & \multicolumn{5}{c}{\scriptsize$\hat{e}_T\!=\!.16$, $\lambda^\star\!=\!2.4$, $\beta\!=\!.38$} & \multicolumn{5}{c}{\scriptsize$\hat{e}_T\!=\!.05$, $\lambda^\star\!=\!4.4$, $\beta\!=\!.22$} & \multicolumn{5}{c@{\tblsep}}{\scriptsize$\hat{e}_T\!=\!.08$, $\lambda^\star\!=\!4.1$, $\beta\!=\!.33$} & \multicolumn{5}{c}{\scriptsize$\hat{e}_T\!=\!.34$, $\lambda^\star\!=\!2.2$, $\beta\!=\!.76$} & \multicolumn{5}{c@{}}{\scriptsize$\hat{e}_T\!=\!.28$, $\lambda^\star\!=\!2.8$, $\beta\!=\!.77$} \\
\cmidrule(lr){2-6}\cmidrule(lr){7-11}\cmidrule(lr){12-16}\cmidrule(r){17-21}\cmidrule(l){22-26}\cmidrule(lr){27-31}
\textbf{Method} & Act & Acc & WA & {\scriptsize WA/$\beta$} & Rnd & Act & Acc & WA & {\scriptsize WA/$\beta$} & Rnd & Act & Acc & WA & {\scriptsize WA/$\beta$} & Rnd & Act & Acc & WA & {\scriptsize WA/$\beta$} & Rnd & Act & Acc & WA & {\scriptsize WA/$\beta$} & Rnd & Act & Acc & WA & {\scriptsize WA/$\beta$} & Rnd \\
\midrule
\rowcolor{blue!6}
\textbf{Ours} & \textbf{27.9} & \textbf{88.9} & \textbf{3.2} & \textbf{.10} & \textbf{3.5} & \textbf{71.1} & \textbf{93.9} & \textbf{4.4} & \textbf{.12} & \textbf{2.2} & \textbf{75.5} & \textbf{97.1} & \textbf{2.3} & \textbf{.11} & \textbf{2.3} & \textbf{84.4} & \textbf{96.4} & \textbf{3.0} & \textbf{.09} & \textbf{1.8} & \color{black!50}9.0 & \color{black!50}70.3 & \color{black!50}2.7 & \color{black!50}.04 & \color{black!50}3.8 & \color{black!50}53.3 & \color{black!50}82.5 & \color{black!50}9.4 & \color{black!50}.12 & \color{black!50}3.1 \\
$k$NN no bias \textit{(ours)} & 61.5 & 85.6 & 8.9 & .29 & 2.4 & 84.3 & 90.4 & 8.1 & .21 & 1.6 & 100 & 95.9 & 4.1 & .19 & 1.0 & 95.7 & 92.7 & 7.1 & .22 & 1.2 & 100 & 64.4 & 35.6 & .47 & 1.0 & 100 & 60.5 & 39.5 & .51 & 1.0 \\
Final-round \textit{(ours)} & 27.2 & 89.0 & 3.1 & .10 & 4.0 & 69.1 & 94.4 & 3.9 & .10 & 4.0 & 62.3 & 96.9 & 1.9 & .09 & 4.0 & 84.4 & 96.4 & 3.1 & .10 & 4.0 & 46.3 & 69.3 & 14.3 & .19 & 4.0 & 60.9 & 82.7 & 10.6 & .14 & 4.0 \\
\midrule
CRC & 100 & 74.4 & 25.6 & .82 & 1.0 & 100 & 79.9 & 20.1 & .53 & 1.0 & 99.9 & 96.0 & 4.0 & .18 & 1.0 & 100 & 91.6 & 8.4 & .26 & 1.1 & 100 & 64.3 & 35.7 & .47 & 1.0 & 99.9 & 63.0 & 36.9 & .48 & 1.2 \\
Selective Pred. & 98.6 & 76.1 & 23.6 & .76 & 1.1 & 97.4 & 85.0 & 14.6 & .38 & 1.3 & 100 & 95.9 & 4.1 & .19 & 1.0 & 99.9 & 91.7 & 8.3 & .25 & 1.1 & 100 & 64.6 & 35.5 & .47 & 1.0 & 98.6 & 68.4 & 31.2 & .41 & 1.3 \\
Calibrated Learned & 97.9 & 76.1 & 23.4 & .75 & 1.2 & 97.3 & 84.5 & 15.1 & .40 & 1.3 & 100 & 95.9 & 4.1 & .19 & 1.0 & 99.7 & 91.4 & 8.6 & .26 & 1.0 & 100 & 64.4 & 35.6 & .47 & 1.0 & 96.6 & 66.1 & 32.9 & .43 & 1.4 \\
Isotonic Conf. & 99.1 & 75.8 & 24.0 & .77 & 1.1 & 98.3 & 84.4 & 15.4 & .40 & 1.3 & 100 & 96.0 & 4.0 & .18 & 1.0 & 99.7 & 91.7 & 8.3 & .25 & 1.1 & 100 & 64.7 & 35.3 & .46 & 1.0 & 96.7 & 70.4 & 28.6 & .37 & 1.3 \\
\midrule
Consensus & 93.6 & 76.0 & 24.0 & .77 & 1.7 & 94.5 & 84.5 & 15.5 & .41 & 1.7 & 99.0 & 95.9 & 4.1 & .19 & 1.1 & 98.2 & 91.5 & 8.5 & .26 & 1.3 & 93.1 & 66.8 & 33.2 & .44 & 1.6 & 87.0 & 71.3 & 28.7 & .37 & 1.8 \\
Confidence Thr. & 33.3 & 76.0 & 24.0 & .77 & 2.5 & 69.8 & 84.5 & 15.5 & .41 & 2.4 & 83.0 & 95.9 & 4.1 & .19 & 1.9 & 87.1 & 91.5 & 8.5 & .26 & 1.6 & 32.3 & 67.2 & 32.8 & .43 & 3.3 & 50.5 & 71.9 & 28.1 & .37 & 2.8 \\
Learned Stopper & 96.6 & 83.6 & 16.4 & .53 & 1.3 & 96.6 & 83.6 & 16.4 & .43 & 1.3 & 100 & 95.7 & 4.3 & .20 & 1.0 & 99.4 & 91.7 & 8.3 & .25 & 1.1 & 98.8 & 63.5 & 36.5 & .48 & 1.1 & --- & --- & --- & --- & --- \\
\midrule
\textit{Oracle} & \textit{78.6} & \textit{100} & \textit{0} & \textit{.00} & \textit{1.7} & \textit{86.0} & \textit{100} & \textit{0} & \textit{.00} & \textit{1.5} & \textit{96.4} & \textit{100} & \textit{0} & \textit{.00} & \textit{1.1} & \textit{92.6} & \textit{100} & \textit{0} & \textit{.00} & \textit{1.2} & \textit{69.8} & \textit{100} & \textit{0} & \textit{.00} & \textit{2.0} & \textit{74.5} & \textit{100} & \textit{0} & \textit{.00} & \textit{1.9} \\
\bottomrule
\end{tabular}}
\renewcommand{\arraystretch}{1.0}
\end{table*}

Table~\ref{tab:main} shows all six datasets ordered by sample size; $\lambda^\star_d$ and $\beta_d$ are determined entirely on training data, with no test-set metric influencing the budget, threshold, or $k$. We do \emph{not} fix a single absolute $\beta$ across datasets---the same absolute budget corresponds to different tolerances depending on final-round error (ARC $\widehat{e}_T\!\approx\!5\%$ vs.\ MuSR $\widehat{e}_T\!\approx\!34\%$)---so we parameterize the deployment tolerance by $\lambda$ and induce $\beta_d$ from training alone; within each dataset all methods are compared under the same $\beta_d$, while $\mathrm{WA}/\beta$ makes tolerance comparable across datasets.

\paragraph{Certifiable datasets.}
On MMLU-Pro, BBH, LogiQA, and ARC the method activates 28--84\% of instances at 89--97\% Acc$|$Act, consuming 9--12\% of the pre-declared budget (WA/$\beta \in [.094, .116]$). CRC and Calibrated Learned achieve near-100\% automation via marginal score-threshold calibration but consume 18--82\% of the same budget. The $k$-NN ablation without $\btk$ acts more aggressively at 2--3$\times$ the usage ratio (confirming the modulus correction is essential); Final-round applies the bound only at $t\!=\!T$ and forgoes all early-stopping savings (stopping round 4.0 vs.\ 1.8--3.5). Fixed-$\beta$ sweeps (Appendix~\ref{app:fixed_beta}) confirm BBH activates at $\beta\!=\!.15$ (75\% Act, 99\% Acc); Pareto frontier in Figure~\ref{fig:pareto}; $\lambda$-sweep in Appendix~\ref{app:budget_sweep}.

\paragraph{Stress-test datasets.}
GPQA and MuSR serve as stress tests: their large induced $\beta\!=\!.77, .76$ are diagnostic signals---small $n\!=\!546, 756$ and high $e_T\!=\!28$--$34\%$ leave calibration insufficient for moderate-risk certification. GPQA activates 53\% at 83\% Acc$|$Act; MuSR activates on only 2/10 splits; the method defers rather than forcing unreliable automation.

\subsection{Matched-automation comparison}
\label{subsec:matched}

Table~\ref{tab:main} could in principle be read as ``our method acts less and therefore looks more accurate''. To isolate the effect of the stopping criterion from the automation rate we run a matched-automation comparison, which we treat as an \emph{oracle-style diagnostic}: for each dataset we fix a target automation $A^\star$ equal to our test-time Act under the training-selected $\lambda^\star$, then grid-search each baseline's threshold on the training split only to minimize $|\mathrm{Act}_{\mathrm{train}} - A^\star|$, and evaluate on the held-out test split (Appendix~\ref{app:matched_auto}). This asks a retrospective question---\emph{if} a scalar threshold were tuned post-hoc to exactly recover our operating point, how would its subset quality compare?---whereas our method constructs that operating point \emph{prospectively} from the declared wrong-action budget and the local lower bound. Table~\ref{tab:matched_main} shows that at matched automation the three stopping rules obtain essentially the same acted-on accuracy ($|\Delta\mathrm{Acc}| \leq 1.6$pp across twelve comparisons) and the same $\mathrm{WA}/\beta$. \emph{The goal is not to outperform tuned scalar thresholds at a matched automation level but to prospectively convert a user-declared wrong-action budget into an auditable act-or-defer operating point before deployment}---at their risk-calibrated defaults, baselines consume 18--82\% of the same budget at near-full automation (Table~\ref{tab:main}) precisely because they cannot convert a declared budget into an automation level without the kind of per-task threshold search used in this diagnostic.

\begin{table}[t]
\caption{\textbf{Matched-automation comparison as an oracle-style diagnostic} (10 random 50/50 splits). Each baseline's threshold is grid-searched on the training split to minimize $|\mathrm{Act}_{\mathrm{train}} - A^\star|$ and evaluated on the held-out test split. $\Delta$ is Acc$|$Act relative to Ours at the matched operating point (pp).}
\label{tab:matched_main}
\centering\renewcommand{\arraystretch}{1.02}
\setlength{\tabcolsep}{3pt}
\resizebox{\textwidth}{!}{%
\begin{tabular}{@{} l c rrrr rrrr rrrr @{}}
\toprule
 & & \multicolumn{4}{c}{\textbf{Ours}} & \multicolumn{4}{c}{\textbf{Confidence Threshold} (matched)} & \multicolumn{4}{c}{\textbf{Calibrated Learned} (matched)} \\
\cmidrule(lr){3-6}\cmidrule(lr){7-10}\cmidrule(lr){11-14}
\textbf{Dataset} & $A^\star$ & Acc & WA & WA/$\beta$ & $\Delta$ & Acc & WA & WA/$\beta$ & $\Delta$ & Acc & WA & WA/$\beta$ & $\Delta$ \\
\midrule
LogiQA   & 27.9\% & 88.9 & 3.18 & .102 & --- & 88.5 & 3.20 & .103 & $-0.4$ & 90.5 & 2.60 & .084 & $+1.6$ \\
MMLU-Pro & 71.1\% & 93.9 & 4.40 & .116 & --- & 93.8 & 4.45 & .117 & $-0.1$ & 94.2 & 4.14 & .109 & $+0.3$ \\
ARC      & 75.5\% & 97.1 & 2.29 & .108 & --- & 96.9 & 2.32 & .112 & $-0.2$ & 96.8 & 2.82 & .135 & $-0.3$ \\
BBH      & 84.4\% & 96.4 & 3.02 & .093 & --- & 96.5 & 2.92 & .090 & $+0.1$ & 96.6 & 2.91 & .090 & $+0.2$ \\
\bottomrule
\end{tabular}}
\renewcommand{\arraystretch}{1.0}
\end{table}

\subsection{Pareto frontier and audit summary}

\begin{figure}[t]
\centering
\includegraphics[width=\textwidth]{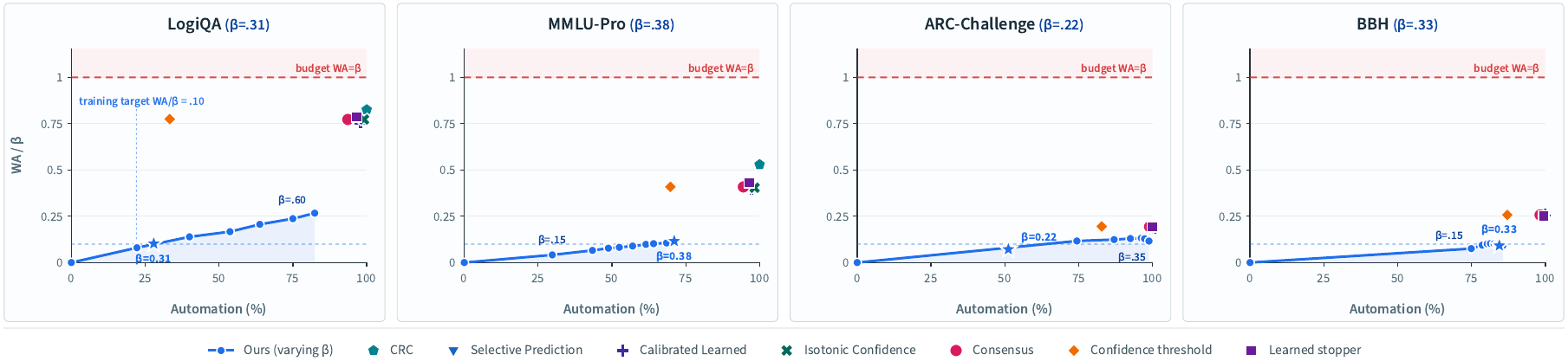}
\caption{\textbf{Risk--automation frontier under normalized budget usage.} $y$-axis: $\mathrm{WA}/\beta$ (fraction of the pre-declared wrong-action budget consumed; lower is safer). Blue: our method as $\beta$ varies. Markers: training-selected baseline operating points from Table~\ref{tab:main} (risk-control, confidence, consensus, and learned-stopper baselines). Dashed red at $1.0$ = budget exhaustion; dotted blue at $0.10$ = training-selection target. Ours stays well below the budget across activated datasets, while many baseline operating points consume substantially more of the pre-declared budget at higher automation.}
\label{fig:pareto}
\end{figure}

Figure~\ref{fig:pareto} shows the continuous risk--automation frontier. The assumption audit (Table~\ref{tab:audit}, Appendix~\ref{app:diagnostics}) targets the two local-approach assumptions---the bias envelope (Assumption~\ref{assum:bias}, via smoothness coverage and $2\times$ modulus inflation) and the representation gap (Assumption~\ref{assum:epsact}, via the predictor-gap proxy)---and assigns a certification-gap value $(1-\alpha)-\widetilde{L}_{\max}$ to each dataset: BBH has a negative gap (strong activation), MMLU-Pro, ARC, and LogiQA have modest positive gaps (partial activation under the training-selected budget), and GPQA/MuSR have large gaps ($>\!0.47$) with small $n$ or high base error, so the certificate defers rather than activating. Inflating the empirical modulus by $2\times$ reduces Act by only 1--3pp on active datasets with Acc$|$Act unchanged, confirming that the modulus estimate already sits well above the true local variation. Full diagnostics, $\epsact$ sensitivity, and budget-split sensitivity over $(\delta, \epsact)$ are in Appendices~\ref{app:diagnostics}--\ref{app:budget_split_sens}.

\section{Discussion and Limitations}
\label{sec:discussion}

The paper does not provide distribution-free conditional correctness; the contribution is to expose assumptions, budget their cost (Proposition~\ref{prop:assumption_cost}), and stress-test them with diagnostic proxies (Table~\ref{tab:audit}). The representation gap (Assumption~\ref{assum:epsact}) is the unavoidable price of a deployable low-dimensional stopping state: any practical system commits to a finite summary of $\Fcal_t$, and $\epsact$ quantifies the one-sided cost of that commitment in the action region; our diagnostics do not prove this gap is small but are designed to falsify obviously insufficient representations and to guide conservative inflation of $\epsact$ (Appendix~\ref{app:rep_gap_stress}). Accordingly $\epsact$ enters the wrong-action bound additively, and the certificate degrades gracefully if Assumption~\ref{assum:epsact} fails (Proposition~\ref{prop:assumption_cost}). The 2D default state $(p_t^{(1)}, \Delta_t)$ empirically balances statistical efficiency and representation loss (Table~\ref{tab:phi_tradeoff}). Low activation on GPQA and MuSR is consistent with the statistical nature of local certification: small $n$ and high base error enlarge both the $k$-nearest-neighbor approximation radius and the concentration slack, so the lower bound rarely crosses threshold under moderate budgets; the per-component decomposition is in Appendix~\ref{app:failure_decomp}. We view this deferral as intended behavior rather than a failure.

\section{Related Work}
\label{sec:related}

\paragraph{Multi-agent debate and LLM confidence.}
Multi-agent debate \citep{du2023improving, liang2024encouraging} improves reasoning but lacks explicit wrong-action budgets; related directions include stability detection \citep{hu2025multi} and self-consistency \citep{wang2023selfconsistency}. LLM confidence estimation \citep{kadavath2022language, xiong2024can, geng2024survey} and calibration \citep{guo2017calibration} focus on aggregate metrics, not local certified bounds.

\paragraph{Conformal, selective prediction, and anytime-valid inference.}
Split conformal \citep{vovk2005algorithmic, angelopoulos2023conformal} provides marginal coverage; extensions include risk control \citep{bates2021distribution, angelopoulos2024conformal} and LLM applications \citep{kumar2023conformal, quach2024conformal}. Localized methods \citep{guan2023localized, hore2023conformal} share our localization but target coverage, not correctness. Key distinctions: (a)~selective prediction \citep{elyaniv2010foundations, geifman2017selective} uses marginal score thresholds without local state-conditional act-or-defer certificates---under the same difficulty-normalized deployment budget, these baselines consume 18--82\% of the pre-declared wrong-action allowance (Table~\ref{tab:main}), while our state-conditional bound extracts high-precision automation subsets at $\leq$12\% on activated datasets; (b)~CRC~\citep{angelopoulos2024conformal} provides aggregate risk control but not adaptive stopping or state-conditional bounds; (c)~localized conformal targets coverage, not correctness. We are not aware of prior work that combines all four elements---adaptive multi-round stopping, local state-conditional certification, wrong-action budgeting with a task-relative-risk deployment policy, and multi-agent deliberation---in a single framework (Appendix~\ref{app:method_comparison}).

\section{Conclusion}
\label{sec:conclusion}

We reframe multi-agent deliberation as budgeted act-or-defer for human-in-the-loop deployment: local $k$-NN lower bounds on state-conditional correctness trigger autonomous action when reliable and defer otherwise, with $\beta = \delta + \alpha + \epsact$. The guarantee is conditional on two auditable structural assumptions (bias envelope and representation gap), each paired with a diagnostic proxy rather than a formal verifier; assumption failures enter the wrong-action bound additively (Proposition~\ref{prop:assumption_cost}), making the cost of each violation explicit rather than hidden behind a global correctness claim. Under training-selected difficulty-normalized budgets, the framework uses 9--12\% of the allowed budget on activated datasets (BBH 84\%/96\%, MMLU-Pro 71\%/94\%, ARC 75\%/97\%, LogiQA 28\%/89\%), while baselines at risk-calibrated defaults consume 18--82\%; on stress tests (GPQA, MuSR) the method defers rather than forcing unreliable automation, the intended behavior under insufficient calibration evidence. The contribution is not better subset quality at a matched automation level---tuned scalar thresholds can recover the same operating point after per-task search---but the prospective conversion of a declared wrong-action budget into an auditable act-or-defer policy before deployment, under explicitly exposed assumptions; we view this as a template for other inference-time compute pipelines where repeated autonomous action is cheap but occasional wrong action is costly. The central open challenge is the geometry--sufficiency trade-off governing $\epsact$, motivating future work on learned representations and extensions to open-ended generation.

\bibliography{references}
\bibliographystyle{plainnat}

\appendix

\section{Proofs}
\label{app:proofs}

\subsection{Proof of Theorem~\ref{thm:pointwise}}

\begin{proof}
Fix a round $t$ and a candidate $k \in \Kcal$. Define the sigma-field $\Sigma_{\mathrm{loc}} = \sigma\!\big(\{U_{i,t}\}_{i=1}^{n},\, U_t\big)$, generated by all calibration states at round $t$ and the test state. Because calibration and test instances are i.i.d.\ (Assumption~\ref{assum:exch}), $\Sigma_{\mathrm{loc}}$ is independent of the correctness labels $\{Z_{i,t}\}_{i=1}^n$ \emph{conditional on the states}---that is, each $Z_{i,t}$ is conditionally independent given $U_{i,t}$ with $\mathbb{E}[Z_{i,t} \mid U_{i,t}] = q_t(U_{i,t})$.

\textbf{Step 1: neighborhood is $\Sigma_{\mathrm{loc}}$-measurable.} Conditional on $\Sigma_{\mathrm{loc}}$, the neighborhood index set $\Ncal_{t,k}(U_t) \subset \{1,\ldots,n\}$ and the $k$-NN radius $h_{t,k}(U_t)$ are deterministic, since they depend only on the distances $\{d(U_{i,t}, U_t)\}$.

\textbf{Step 2: bias envelope controls the conditional mean.} By the bias envelope (Assumption~\ref{assum:bias}):
\begin{equation}
    \frac{1}{k}\sum_{i \in \Ncal_{t,k}(U_t)} q_t(U_{i,t}) \;\geq\; q_t(U_t) - \btk(U_t).
\end{equation}
Therefore, the conditional mean of the $k$-NN estimator satisfies $\mathbb{E}[\widehat{q}_{t,k}(U_t) \mid \Sigma_{\mathrm{loc}}] = \frac{1}{k}\sum_{i \in \Ncal_{t,k}(U_t)} q_t(U_{i,t}) \geq q_t(U_t) - \btk(U_t)$.

\textbf{Step 3: upper-tail Hoeffding conditional on $\Sigma_{\mathrm{loc}}$.} Since $Z_{i,t} \in [0,1]$ are conditionally independent given $\Sigma_{\mathrm{loc}}$ (each $Z_{i,t}$ depends on its own instance, and cross-instance independence follows from Assumption~\ref{assum:exch}), the upper-tail Hoeffding inequality \citep{hoeffding1963probability} gives:
\begin{equation}
    \Pr\!\left(\widehat{q}_{t,k}(U_t) > \mu_{t,k} + \sqrt{\frac{\log(1/\delta_{t,k})}{2k}} \;\middle|\; \Sigma_{\mathrm{loc}}\right) \leq \delta_{t,k},
\end{equation}
where $\mu_{t,k} := \mathbb{E}[\widehat{q}_{t,k}(U_t) \mid \Sigma_{\mathrm{loc}}] = \frac{1}{k}\sum_{i \in \Ncal_{t,k}(U_t)} q_t(U_{i,t})$.

\textbf{Step 4: combine and integrate.} Write $c_k = \sqrt{\log(1/\delta_{t,k})/(2k)}$. On the event $\{\widehat{q}_{t,k} \leq \mu_{t,k} + c_k\}$:
\begin{equation}
    L_{t,k}(U_t) = \widehat{q}_{t,k}(U_t) - \btk(U_t) - c_k \leq \mu_{t,k} + c_k - \btk(U_t) - c_k = \mu_{t,k} - \btk(U_t) \leq q_t(U_t),
\end{equation}
where the last step uses $\mu_{t,k} \leq q_t(U_t) + \btk(U_t)$ from the bias envelope (Assumption~\ref{assum:bias}). Hence $\Pr(E_{t,k} \mid \Sigma_{\mathrm{loc}}) \leq \delta_{t,k}$ almost surely. Integrating over $\Sigma_{\mathrm{loc}}$:
\begin{equation}
    \Pr(E_{t,k}) = \mathbb{E}\big[\Pr(E_{t,k} \mid \Sigma_{\mathrm{loc}})\big] \leq \delta_{t,k}. \qedhere
\end{equation}
\end{proof}

\begin{remark}[Reuse of calibration data across $(t,k)$ pairs]
The same calibration instances are used for every round $t$ and every $k \in \Kcal$. This does \emph{not} invalidate the conditional Hoeffding argument, because we condition on $\Sigma_{\mathrm{loc}}$ separately for each $(t,k)$ and then combine via a union bound (Theorem~\ref{thm:simultaneous}). Within any single $(t,k)$, the neighborhood is deterministic given $\Sigma_{\mathrm{loc}}$, and the relevant $Z_{i,t}$ are conditionally independent.
\end{remark}

\subsection{Proof of Theorem~\ref{thm:simultaneous}}

\begin{proof}
Apply Theorem~\ref{thm:pointwise} at each $(t, k)$ pair with $\delta_{t,k} = \delta / (T|\Kcal|)$. Union bound:
\begin{equation}
    \Pr\!\left(\bigcup_{t \leq T}\bigcup_{k \in \Kcal} E_{t,k}\right) \leq \sum_{t=1}^T \sum_{k \in \Kcal} \Pr(E_{t,k}) \leq \sum_{t=1}^T \sum_{k \in \Kcal} \delta_{t,k} = \delta.
\end{equation}
On the complement $\Gcal = \bigcap_{t,k} E_{t,k}^c$, every $L_{t,k}(U_t) \leq q_t(U_t)$, so $L_t = \max_k L_{t,k}(U_t) \leq q_t(U_t)$ as well.
\end{proof}

\subsection{Proof of Theorem~\ref{thm:wrong_action}}

\begin{proof}
Let $\Gcal = \bigcap_{t,k} E_{t,k}^c = \{\forall\, t \leq T,\; \forall\, k \in \Kcal : L_{t,k}(U_t) \leq q_t(U_t)\}$. By Theorem~\ref{thm:simultaneous}, $\Pr(\Gcal^c) \leq \delta$.

\textbf{Step 1: decompose over the good event.}
\begin{align}
    \Pr(\tau \leq T,\; \Yhat_\tau \neq Y)
    &= \Pr(\tau \leq T,\; \Yhat_\tau \neq Y,\; \Gcal)
    + \Pr(\tau \leq T,\; \Yhat_\tau \neq Y,\; \Gcal^c) \notag \\
    &\leq \Pr(\tau \leq T,\; \Yhat_\tau \neq Y,\; \Gcal) + \delta.
\end{align}

\textbf{Step 2: handle the random stopping time $\tau$.} Let $\mathcal{H}_t = \sigma(\Dcal_{\mathrm{cal}}, \Fcal_t)$ be the augmented filtration containing both the calibration set and the test debate prefix. Since $L_t$ is $\mathcal{H}_t$-measurable (it depends on $\Dcal_{\mathrm{cal}}$ and $\Fcal_t$ through $U_t$), $\tau = \inf\{t : L_t \geq 1 - \alpha\}$ is an $(\mathcal{H}_t)$-stopping time. On $\{\tau \leq T\} \cap \Gcal$, we have $L_\tau(U_\tau) \geq 1 - \alpha$ (stopping rule) and $L_\tau(U_\tau) \leq q_\tau(U_\tau)$ (validity on $\Gcal$), so $q_\tau(U_\tau) \geq 1 - \alpha$, placing the random pair $(\tau, U_\tau)$ in the stopping region of Assumption~\ref{assum:epsact}.

\textbf{Step 3: apply the representation gap.} Assumption~\ref{assum:epsact} is stated in terms of $\Pr(Y = \Yhat_t \mid \Fcal_t = f_t)$. Since $\Dcal_{\mathrm{cal}}$ is independent of the test instance $(X, Y)$ by Assumption~\ref{assum:exch}, we have $\Pr(Y = \Yhat_t \mid \mathcal{H}_t) = \Pr(Y = \Yhat_t \mid \Fcal_t)$ a.s., so the representation gap applies equally under the augmented filtration. On $\{\tau \leq T\} \cap \Gcal$, where $q_\tau(U_\tau) \geq 1 - \alpha$:
\begin{equation}
    \Pr(Y = \Yhat_\tau \mid \Fcal_\tau) \;\geq\; q_\tau(U_\tau) - \epsact \;\geq\; (1 - \alpha) - \epsact.
\end{equation}
Therefore $\Pr(\Yhat_\tau \neq Y \mid \Fcal_\tau) \leq \alpha + \epsact$ on $\{\tau \leq T\} \cap \Gcal$.

\textbf{Step 4: integrate.} Taking expectations over $\Fcal_\tau$:
\begin{align}
    \Pr(\tau \leq T,\; \Yhat_\tau \neq Y,\; \Gcal)
    &= \mathbb{E}\big[\mathbf{1}_{\{\tau \leq T\} \cap \Gcal}\;\Pr(\Yhat_\tau \neq Y \mid \Fcal_\tau)\big] \notag \\
    &\leq (\alpha + \epsact)\,\Pr(\tau \leq T,\; \Gcal)
    \leq \alpha + \epsact.
\end{align}
Combining with Step~1 gives~\eqref{eq:wrong_action_bound}.
\end{proof}

\subsection{Proof of Proposition~\ref{prop:assumption_cost}}
\begin{proof}
Let $\mathcal{E}_{\mathrm{env}}$ be the event that Assumption~\ref{assum:bias} holds for all $(t,k)$ in the action region, and $\mathcal{E}_{\mathrm{rep}}$ the event that Assumption~\ref{assum:epsact} holds. Define $\mathcal{E} = \mathcal{E}_{\mathrm{env}} \cap \mathcal{E}_{\mathrm{rep}}$. Partition the wrong-action probability over $\mathcal{E}$ and its complement:
\begin{align}
    \Pr(\tau \leq T,\; \Yhat_\tau \neq Y)
    &= \Pr(\tau \leq T,\; \Yhat_\tau \neq Y,\; \mathcal{E})
    + \Pr(\tau \leq T,\; \Yhat_\tau \neq Y,\; \mathcal{E}^c). \label{eq:cost_decomp}
\end{align}
\textbf{First term.} On $\mathcal{E}$, both Assumptions~\ref{assum:bias} and~\ref{assum:epsact} hold. Since these are the only non-standard assumptions required by Theorem~\ref{thm:wrong_action} (Assumptions~\ref{assum:exch}--\ref{assum:measurable} hold unconditionally by the experimental design), the proof of Theorem~\ref{thm:wrong_action} applies \emph{restricted to $\mathcal{E}$}. Specifically, the Hoeffding argument yields $\Pr(\Gcal^c \cap \mathcal{E}) \leq \delta$ (the good event $\Gcal$ requires only Assumption~\ref{assum:bias}, which holds on $\mathcal{E}$), and Steps~2--4 of the proof yield $\Pr(\tau \leq T,\;\Yhat_\tau \neq Y,\;\Gcal,\;\mathcal{E}) \leq \alpha + \epsact$. Hence:
\begin{equation}
    \Pr(\tau \leq T,\; \Yhat_\tau \neq Y,\; \mathcal{E}) \leq \delta + \alpha + \epsact.
\end{equation}
\textbf{Second term.} By the union bound on the complement:
\begin{equation}
    \Pr(\tau \leq T,\;\Yhat_\tau \neq Y,\;\mathcal{E}^c) \leq \Pr(\mathcal{E}^c) = \Pr(\mathcal{E}_{\mathrm{env}}^c \cup \mathcal{E}_{\mathrm{rep}}^c) \leq \Pr(\mathcal{E}_{\mathrm{env}}^c) + \Pr(\mathcal{E}_{\mathrm{rep}}^c).
\end{equation}
Combining with~\eqref{eq:cost_decomp} gives the result. The bound degrades \emph{gracefully}: each assumption violation adds exactly its failure probability, and no catastrophic interaction occurs between the two failure modes.
\end{proof}

\subsection{Proof of Corollary~\ref{cor:acted_error}}
\begin{proof}
Let $p_{\mathrm{act}} = \Pr(\tau \leq T) > 0$. By definition of conditional probability:
\begin{equation}
    \Pr(\Yhat_\tau \neq Y \mid \tau \leq T) = \frac{\Pr(\tau \leq T,\; \Yhat_\tau \neq Y)}{\Pr(\tau \leq T)}.
\end{equation}
We decompose the numerator over the good event $\Gcal = \bigcap_{t,k} E_{t,k}^c$:
\begin{align}
    \Pr(\tau \leq T,\;\Yhat_\tau \neq Y) &= \Pr(\tau \leq T,\;\Yhat_\tau \neq Y,\;\Gcal) + \Pr(\tau \leq T,\;\Yhat_\tau \neq Y,\;\Gcal^c).
\end{align}
\textbf{First term.} On $\Gcal \cap \{\tau \leq T\}$, the stopping rule ensures $L_\tau(U_\tau) \geq 1 - \alpha$, and validity on $\Gcal$ gives $q_\tau(U_\tau) \geq L_\tau(U_\tau) \geq 1-\alpha$. By Assumption~\ref{assum:epsact}, $\Pr(Y = \Yhat_\tau \mid \Fcal_\tau) \geq q_\tau(U_\tau) - \epsact \geq 1 - \alpha - \epsact$, so $\Pr(\Yhat_\tau \neq Y \mid \Fcal_\tau) \leq \alpha + \epsact$ a.s.\ on this event. Taking expectations:
\begin{align}
    \Pr(\tau \leq T,\;\Yhat_\tau \neq Y,\;\Gcal)
    &= \mathbb{E}\big[\mathbf{1}_{\{\tau \leq T\}\cap\Gcal}\,\Pr(\Yhat_\tau \neq Y \mid \Fcal_\tau)\big] \notag \\
    &\leq (\alpha + \epsact)\,\Pr(\tau \leq T,\;\Gcal) \leq (\alpha + \epsact)\, p_{\mathrm{act}}.
\end{align}
\textbf{Second term.} By Theorem~\ref{thm:simultaneous}, $\Pr(\Gcal^c) \leq \delta$, so $\Pr(\tau \leq T,\;\Yhat_\tau \neq Y,\;\Gcal^c) \leq \delta$.

Combining and dividing by $p_{\mathrm{act}}$:
\begin{equation}
    \Pr(\Yhat_\tau \neq Y \mid \tau \leq T) \leq \frac{(\alpha + \epsact)\, p_{\mathrm{act}} + \delta}{p_{\mathrm{act}}} = \alpha + \epsact + \frac{\delta}{p_{\mathrm{act}}}. \qedhere
\end{equation}
\end{proof}

\subsection{Proof of Theorem~\ref{thm:oracle_slack}}

\begin{theorem}[Oracle slack inequality]
\label{thm:oracle_slack}
Under Assumptions~\ref{assum:exch}--\ref{assum:bias}, with $c_k = \sqrt{\log(T|\Kcal|/\delta)/(2k)}$ and oracle $k^{\mathrm{or}}_t = \arg\min_{k \in \Kcal}[\btk(U_t) + c_k]$, the two-sided event $\Gcal_2 = \{\forall\, t,k : |\widehat{q}_{t,k} - \mu_{t,k}| \leq c_k\}$ satisfies $\Pr(\Gcal_2) \geq 1 - 2\delta$, and on $\Gcal_2$:
\begin{equation}
\label{eq:oracle_slack}
    q_t(U_t) - L_t(U_t) \;\leq\; 2\big[b_{t,k^{\mathrm{or}}_t}(U_t) + c_{k^{\mathrm{or}}_t}\big] \quad \forall\, t \leq T.
\end{equation}
\end{theorem}

\begin{proof}
Write $c_k = \sqrt{\log(T|\Kcal|/\delta)/(2k)}$ and $\mu_{t,k} = \frac{1}{k}\sum_{i \in \Ncal_{t,k}(U_t)} q_t(U_{i,t})$ for the conditional mean of $\widehat{q}_{t,k}$.

\textbf{Probability bound.} Conditional on $\Sigma_{\mathrm{loc}}$, two-sided Hoeffding gives $\Pr(|\widehat{q}_{t,k} - \mu_{t,k}| > c_k \mid \Sigma_{\mathrm{loc}}) \leq 2\exp(-2k c_k^2) = 2\delta/(T|\Kcal|)$. Union-bounding over all $T|\Kcal|$ pairs and integrating: $\Pr(\Gcal_2^c) \leq 2\delta$. Since the upper tail of $\Gcal_2$ implies $\widehat{q}_{t,k} \leq \mu_{t,k} + c_k$, which (combined with the bias envelope) gives $L_{t,k} \leq q_t$, we have $\Gcal_2 \subseteq \Gcal$.

\textbf{Slack bound on $\Gcal_2$.} On $\Gcal_2$, the lower tail gives:
\begin{equation}
\label{eq:oracle_hoeffding}
    \widehat{q}_{t,k}(U_t) \;\geq\; \mu_{t,k} - c_k \quad \forall\, t \leq T,\; k \in \Kcal.
\end{equation}
By Assumption~\ref{assum:bias}, $\mu_{t,k} \geq q_t(U_t) - b_{t,k}(U_t)$. Combining:
\begin{equation}
    \widehat{q}_{t,k}(U_t) \;\geq\; q_t(U_t) - b_{t,k}(U_t) - c_k.
\end{equation}
Substituting into the definition $L_{t,k}(U_t) = \widehat{q}_{t,k}(U_t) - b_{t,k}(U_t) - c_k$:
\begin{equation}
    L_{t,k}(U_t) \;\geq\; q_t(U_t) - 2\big[b_{t,k}(U_t) + c_k\big].
\end{equation}
This holds for every $k \in \Kcal$. Since $L_t(U_t) = \max_{k \in \Kcal} L_{t,k}(U_t) \geq L_{t,k^{\mathrm{or}}_t}(U_t)$, we obtain
\begin{equation}
    q_t(U_t) - L_t(U_t) \;\leq\; q_t(U_t) - L_{t,k^{\mathrm{or}}_t}(U_t) \;\leq\; 2\big[b_{t,k^{\mathrm{or}}_t}(U_t) + c_{k^{\mathrm{or}}_t}\big],
\end{equation}
where the last step applies the preceding inequality at $k = k^{\mathrm{or}}_t$. Since $k^{\mathrm{or}}_t$ minimizes $\btk(U_t) + c_k$ over $\Kcal$, this is the smallest possible right-hand side, giving~\eqref{eq:oracle_slack}.
\end{proof}

\subsection{Proofs of instantiation corollaries}

The generic bias envelope $\btk$ specializes to classical regularity models.

\begin{corollary}[Lipschitz envelope]
\label{cor:lipschitz_det}
If $q_t$ is $L_t^\star$-Lipschitz, then $\btk(u) = L_t^\star \cdot h_{t,k}(u)$ satisfies Assumption~\ref{assum:bias} deterministically.
\end{corollary}

\begin{corollary}[Estimated modulus envelope]
\label{cor:modulus}
If $\omegahat_t$ is estimated on a separate split with $\Pr(\omegahat_t \geq \omega_t) \geq 1 - \delta_\omega$, the failure probability $\delta_\omega$ is absorbed into $\delta$, keeping $\beta = \delta + \alpha + \epsact$ intact.
\end{corollary}

\begin{corollary}[H\"{o}lder instantiation]
\label{cor:holder}
If $|q_t(u) - q_t(v)| \leq C_t \cdot d(u,v)^s$, setting $\btk(u) = C_t \cdot h_{t,k}(u)^s$ yields a valid envelope with rate $O(n^{-s/(2s+D)})$.
\end{corollary}

\begin{proof}[Proof of Corollary~\ref{cor:lipschitz_det} (Lipschitz envelope)]
Let $q_t$ be $L_t^\star$-Lipschitz: $|q_t(u) - q_t(v)| \leq L_t^\star \cdot d(u,v)$ for all $u,v \in \Ucal$. Fix $k \in \Kcal$ and a query point $u$. For every $i \in \Ncal_{t,k}(u)$, the neighbor $U_{i,t}$ satisfies $d(U_{i,t}, u) \leq h_{t,k}(u)$ by definition of the $k$-NN radius. Applying the Lipschitz condition and taking the average:
\begin{align}
    \bigg|\frac{1}{k}\sum_{i \in \Ncal_{t,k}(u)} q_t(U_{i,t}) - q_t(u)\bigg|
    &\leq \frac{1}{k}\sum_{i \in \Ncal_{t,k}(u)} |q_t(U_{i,t}) - q_t(u)| \\
    &\leq \frac{1}{k}\sum_{i \in \Ncal_{t,k}(u)} L_t^\star \cdot d(U_{i,t}, u) \;\leq\; L_t^\star \cdot h_{t,k}(u).
\end{align}
This holds for every realization of the calibration and test data---no probabilistic argument is needed. Hence $\btk(u) = L_t^\star \cdot h_{t,k}(u)$ satisfies Assumption~\ref{assum:bias} deterministically.
\end{proof}

\begin{proof}[Proof of Corollary~\ref{cor:modulus} (estimated modulus envelope)]
The empirical modulus $\omegahat_t$ is estimated on the separate split $\Dcal_{\mathrm{mod}}$ via~\eqref{eq:modulus}. Let $\omega_t(r) = \sup_{d(u,v)\leq r} |q_t(u) - q_t(v)|$ be the true modulus of continuity of $q_t$. By assumption, $\Pr(\forall\, r : \omegahat_t(r) \geq \omega_t(r)) \geq 1 - \delta_\omega$.

\textbf{On the event $\{\omegahat_t \geq \omega_t\}$:} For any neighbor $U_{i,t} \in \Ncal_{t,k}(u)$ with $d(U_{i,t}, u) \leq h_{t,k}(u)$, the definition of the modulus gives $|q_t(U_{i,t}) - q_t(u)| \leq \omega_t(h_{t,k}(u)) \leq \omegahat_t(h_{t,k}(u))$. By the same averaging argument as Corollary~\ref{cor:lipschitz_det}:
\begin{equation}
    \bigg|\frac{1}{k}\sum_{i \in \Ncal_{t,k}(u)} q_t(U_{i,t}) - q_t(u)\bigg| \leq \omegahat_t(h_{t,k}(u)) = \btk(u),
\end{equation}
so Assumption~\ref{assum:bias} holds on this event.

\textbf{On the complementary event $\{\exists\, r : \omegahat_t(r) < \omega_t(r)\}$:} Assumption~\ref{assum:bias} may fail. By Proposition~\ref{prop:assumption_cost}, this adds at most $\Pr(\omegahat_t < \omega_t) \leq \delta_\omega$ to the wrong-action bound.

\textbf{Budget accounting.} We allocate the overall $\delta$ as follows: set the Hoeffding level to $\delta - \delta_\omega$ (for the union bound over $(t,k)$ pairs in Theorem~\ref{thm:simultaneous}) and reserve $\delta_\omega$ for modulus estimation failure. The total calibration failure probability remains:
\begin{equation}
    (\delta - \delta_\omega) + \delta_\omega = \delta,
\end{equation}
so the three-way budget $\beta = \delta + \alpha + \epsact$ is preserved. Importantly, $\Dcal_{\mathrm{mod}}$ and $\Dcal_{\mathrm{loc}}$ are independent (separate split), so the modulus $\omegahat_t$ is a fixed function conditional on $\Dcal_{\mathrm{mod}}$. This does not interfere with the Hoeffding conditional independence argument in Theorem~\ref{thm:pointwise}, which operates on $\Dcal_{\mathrm{loc}}$.
\end{proof}

\begin{proof}[Proof of Corollary~\ref{cor:holder} (H\"{o}lder instantiation)]
Let $|q_t(u) - q_t(v)| \leq C_t \cdot d(u,v)^s$ for all $u,v \in \Ucal$ with exponent $s \in (0,1]$. By the same argument as Corollary~\ref{cor:lipschitz_det}, replacing the Lipschitz bound $L_t^\star \cdot d(U_{i,t}, u)$ with the H\"{o}lder bound $C_t \cdot d(U_{i,t}, u)^s$, we obtain:
\begin{equation}
    \bigg|\frac{1}{k}\sum_{i \in \Ncal_{t,k}(u)} q_t(U_{i,t}) - q_t(u)\bigg| \leq C_t \cdot h_{t,k}(u)^s,
\end{equation}
so $\btk(u) = C_t \cdot h_{t,k}(u)^s$ satisfies Assumption~\ref{assum:bias} deterministically.

\textbf{Convergence rate.} Under Assumption~\ref{assum:support}, Theorem~\ref{thm:radius} gives $h_{t,k}(u) \leq (2k/(nc_0))^{1/D}$ with high probability. The oracle slack (Theorem~\ref{thm:oracle_slack}) then satisfies:
\begin{equation}
    q_t(U_t) - L_t(U_t) \leq 2\bigg[C_t\Big(\frac{2k}{nc_0}\Big)^{s/D} + \sqrt{\frac{\log(T|\Kcal|/\delta)}{2k}}\bigg].
\end{equation}
Optimizing over $k$: equating the two terms gives $C_t(k/n)^{s/D} \asymp k^{-1/2}$, yielding $k^\star \asymp n^{2s/(2s+D)}$. Substituting:
\begin{equation}
    q_t(U_t) - L_t(U_t) = O\!\left(n^{-s/(2s+D)}\right),
\end{equation}
which recovers the minimax nonparametric rate for $s$-H\"{o}lder functions in $D$ dimensions \citep{tsybakov2009introduction}. The Lipschitz case ($s = 1$) gives $O(n^{-1/(2+D)})$ as stated in Corollary~\ref{cor:rate}.
\end{proof}

\subsection{Convergence rate results}

\begin{assumption}[Local support (for rate only)]
\label{assum:support}
There exist $c_0 > 0$ and $D > 0$ such that $\Pr(d(U_t, u) \leq r) \geq c_0 r^D$ for all $t$, $u \in \Ucal$, and sufficiently small $r > 0$.
\end{assumption}

\begin{theorem}[$k$-NN radius bound]
\label{thm:radius}
Under Assumption~\ref{assum:support}, for any $t$ and $u$, define $r_{n,k} = (2k/(n\, c_0))^{1/D}$. Then $\Pr(h_{t,k}(u) > r_{n,k}) \leq e^{-k/8}$.
\end{theorem}

\begin{proof}
Fix $t$ and $u$. Let $N_t(u, r) = \sum_{i=1}^{n} \mathbf{1}[d(U_{i,t}, u) \leq r]$. By Assumption~\ref{assum:support} and independence (Assumption~\ref{assum:exch}), $\mathbb{E}[N_t(u, r_{n,k})] \geq n\, c_0\, r_{n,k}^D = 2k$. The multiplicative Chernoff bound gives
\begin{equation}
    \Pr(N_t(u, r_{n,k}) < k) \leq \exp(-k/4) \leq e^{-k/8}.
\end{equation}
On the complement, the $k$-th nearest neighbor is within $r_{n,k}$, so $h_{t,k}(u) \leq r_{n,k}$.
\end{proof}

\subsection{Proof of Corollary~\ref{cor:rate}}

\begin{corollary}[Adaptive certification rate]
\label{cor:rate}
With Lipschitz envelope and $\Kcal$ containing $k^\star \asymp n^{2/(2+D)}$, the slack $q_t - L_t = O(n^{-1/(2+D)})$ with high probability.
\end{corollary}

\begin{proof}
Let $\Gcal_1 = \{\forall\, t, k : L_{t,k}(U_t) \leq q_t(U_t)\}$ and $\Gcal_2 = \{\forall\, t, k : h_{t,k}(U_t) \leq r_{n,k}\}$. By Theorem~\ref{thm:simultaneous}, $\Pr(\Gcal_1^c) \leq \delta$. By Theorem~\ref{thm:radius} with a union bound over $T \cdot |\Kcal|$ pairs, $\Pr(\Gcal_2^c) \leq T|\Kcal|\,e^{-k_{\min}/8}$. On $\Gcal_1 \cap \Gcal_2$, using a Lipschitz bias envelope $\btk(u) = \Lbar_t \cdot h_{t,k}(u)$ (the rate statement also holds for any bias envelope $\btk(u) \leq \Lbar_t \cdot h_{t,k}(u)$, including the modulus instantiation when $\omega_t$ grows at most linearly):
\begin{equation}
    q_t(U_t) - L_t(U_t) \leq \min_{k \in \Kcal}\left[\Lbar_t\!\left(\frac{2k}{n\, c_0}\right)^{1/D} + \sqrt{\frac{\log(T|\Kcal|/\delta)}{2k}}\right].
\end{equation}
A geometric grid $\Kcal$ containing $k^\star \asymp n^{2/(2+D)}$ makes the minimum $O(n^{-1/(2+D)})$.
\end{proof}

\subsection{Group-conditional control}

\begin{theorem}[Group-conditional control]
\label{thm:group}
Under Assumptions~\ref{assum:exch}--\ref{assum:epsact}, per-group bounds $L_t^{(g)}$ at level $\delta_g = \delta/J$ satisfy
$\Pr(\tau \leq T, \Yhat_\tau \neq Y, G(\Fcal_\tau) = g) \leq \delta_g + \alpha + \epsact$.
\end{theorem}

\begin{proof}
Let $G : \Xcal \to \{1,\ldots,J\}$ be a group assignment function determined by the input $X$ (e.g., a demographic attribute or problem category). Since $G$ depends only on $X$ and the calibration instances are i.i.d.\ (Assumption~\ref{assum:exch}), the within-group calibration subsets $\Dcal_{\mathrm{cal}}^{(g)} = \{(X_i, Y_i) : G(X_i) = g\}$ inherit exchangeability: conditional on the group sizes $\{n_g\}_{g=1}^J$, the instances within each group remain i.i.d.\ draws from the group-conditional distribution $P_{X,Y \mid G = g}$.

\textbf{Per-group construction.} For each group $g$, construct per-group $k$-NN neighborhoods $\Ncal_{t,k}^{(g)}(u)$ using only the $n_g$ calibration instances in group $g$, and define:
\begin{equation}
    L_{t,k}^{(g)}(u) = \widehat{q}_{t,k}^{(g)}(u) - b_{t,k}^{(g)}(u) - \sqrt{\frac{\log(T|\Kcal|/\delta_g)}{2k}},
\end{equation}
with Hoeffding level $\delta_g = \delta / J$.

\textbf{Validity within group.} Fix group $g$. The proof of Theorem~\ref{thm:pointwise} applies verbatim within group $g$: the sigma-field $\Sigma_{\mathrm{loc}}^{(g)}$ is generated by group-$g$ states; the correctness labels $\{Z_{i,t}\}_{i \in \text{group } g}$ are conditionally independent given $\Sigma_{\mathrm{loc}}^{(g)}$ (by within-group exchangeability); and Hoeffding yields $\Pr(E_{t,k}^{(g)}) \leq \delta_g / (T|\Kcal|)$. A union bound over $(t,k)$ pairs gives $\Pr(\Gcal_g^c) \leq \delta_g$.

\textbf{Per-group wrong-action.} On $\Gcal_g$, $L_t^{(g)}(U_t) \leq q_t^{(g)}(U_t)$ for all $t$, where $q_t^{(g)}$ is the group-conditional correctness. Applying the same stopping-time and representation-gap arguments as Theorem~\ref{thm:wrong_action} within group $g$:
\begin{equation}
    \Pr(\tau \leq T,\; \Yhat_\tau \neq Y,\; G(\Fcal_\tau) = g) \leq \delta_g + \alpha + \epsact.
\end{equation}

\textbf{Simultaneous control across all groups.} By the union bound:
\begin{equation}
    \sum_{g=1}^J \Pr(\tau \leq T,\; \Yhat_\tau \neq Y,\; G(\Fcal_\tau) = g) \leq \sum_{g=1}^J (\delta_g + \alpha + \epsact) = \delta + J(\alpha + \epsact).
\end{equation}
The per-group statement is the useful one: it guarantees risk control \emph{within each group separately}, at the cost of a tighter Hoeffding term $\sqrt{\log(T|\Kcal|J/\delta)/(2k)}$ (from the reduced per-group $\delta_g$). When $J$ is large or group sizes $n_g$ are small, the per-group $k$-NN radius increases and automation decreases---the price of simultaneous group-conditional control.
\end{proof}

\subsection{Auxiliary results}

\begin{lemma}[Stopping time measurability]
\label{lem:stopping_time}
The stopping rule $\tau = \inf\{t \leq T : L_t(U_t) \geq 1-\alpha\}$ is an $(\mathcal{H}_t)$-stopping time, where $\mathcal{H}_t = \sigma(\Dcal_{\mathrm{cal}}, \Fcal_t)$.
\end{lemma}
\begin{proof}
For each $t$, $L_t(U_t) = \max_{k \in \Kcal} L_{t,k}(U_t)$ is $\mathcal{H}_t$-measurable since each $L_{t,k}$ depends only on $\Dcal_{\mathrm{cal}}$ (the calibration data, which generates a sub-sigma-field of every $\mathcal{H}_s$ for $s \geq 1$) and $U_t = \phi_t(\Fcal_t)$ (which is $\Fcal_t$-measurable by Assumption~\ref{assum:measurable}, hence $\mathcal{H}_t$-measurable). Therefore $\{L_t \geq 1-\alpha\} \in \mathcal{H}_t$, and:
\begin{equation}
    \{\tau \leq t\} = \bigcup_{s=1}^{t} \{L_s(U_s) \geq 1-\alpha\} \in \mathcal{H}_t,
\end{equation}
where the inclusion holds because each $\{L_s \geq 1-\alpha\} \in \mathcal{H}_s \subseteq \mathcal{H}_t$ for $s \leq t$ (the filtration is non-decreasing). Hence $\tau$ is an $(\mathcal{H}_t)$-stopping time.
\end{proof}

\begin{remark}[Bernstein refinement]
\label{rem:bernstein}
The Hoeffding term $c_k = \sqrt{\log(T|\Kcal|/\delta)/(2k)}$ can be tightened using the empirical Bernstein inequality \citep{maurer2009empirical}. Let $\widehat{\sigma}_{t,k}^2 = \widehat{q}_{t,k}(1 - \widehat{q}_{t,k})$ be the plug-in variance estimate. The Bernstein bound replaces $c_k$ with:
\begin{equation}
    c_k^{\mathrm{Bern}} = \sqrt{\frac{2\widehat{\sigma}_{t,k}^2 \log(2T|\Kcal|/\delta)}{k}} + \frac{7\log(2T|\Kcal|/\delta)}{3(k-1)}.
\end{equation}
When the local correctness is near 0 or 1 (i.e., $q_t(U_t)$ is close to the boundary), $\widehat{\sigma}_{t,k}^2$ is small and $c_k^{\mathrm{Bern}} \ll c_k$. This does not change the formal theory---any valid concentration inequality can be substituted---but can improve practical automation in the high-confidence regime. We use Hoeffding as the default for simplicity and because the Bernstein gain is modest in our setting (most acted-on instances have $q_t \in [0.85, 0.99]$ where $\sigma^2$ is not negligibly small).
\end{remark}

\begin{remark}[Sharpness of the union bound]
\label{rem:union_sharpness}
The Bonferroni correction $\delta_{t,k} = \delta/(T|\Kcal|)$ is worst-case tight: an adversary can construct a distribution where the bad events $\{E_{t,k}\}$ are disjoint, making the union bound an equality. In practice, the $L_{t,k}$ values are positively correlated across $k$ (larger $k$ neighborhoods contain smaller ones), so the effective number of independent tests is less than $T|\Kcal|$. A Holm-type sequential correction could exploit this, but at the cost of requiring ordered $p$-values and sequential computation. For our setting with $T|\Kcal| = 4 \times 3 = 12$, the $\log(12/\delta)$ penalty is mild and the Bonferroni correction is practical.
\end{remark}

\begin{remark}[Relationship between Assumptions~\ref{assum:bias} and~\ref{assum:epsact}]
\label{rem:assumption_relationship}
The bias envelope (Assumption~\ref{assum:bias}) and the representation gap (Assumption~\ref{assum:epsact}) control different error sources and are not redundant. To see this, consider two extreme state embeddings. If $\phi_t = \Fcal_t$ (the full debate transcript), then $q_t(\phi_t(\Fcal_t)) = \Pr(Y = \Yhat_t \mid \Fcal_t)$ and $\epsact = 0$, but $\Ucal$ is infinite-dimensional and the bias envelope $\btk$ is uncontrollable (the $k$-NN radius does not shrink). Conversely, if $\phi_t \equiv c$ (a constant embedding), then $D = 0$ and $\btk = 0$ trivially, but $\epsact$ equals the full variation of $\Pr(Y = \Yhat_t \mid \Fcal_t)$ in the action region. Our 2D embedding $(p_t^{(1)}, \Delta_t)$ balances these extremes (Table~\ref{tab:phi_tradeoff}).
\end{remark}

\section{Implementation Details}
\label{app:details}

\begin{algorithm}[h]
\caption{Adaptive Local $k$-NN Certification for Multi-Agent Debate}
\label{alg:main}
\begin{algorithmic}[1]
\REQUIRE Budget $\beta$; calibration $\Dcal_{\mathrm{cal}}$; test input $x$; candidate family $\Kcal$; budget split $(\delta, \epsact)$; bias envelope $\btk(\cdot)$.
\STATE Set $\alpha \leftarrow \beta - \delta - \epsact$.
\FOR{$t = 1, \ldots, T$}
    \STATE Run debate round $t$; compute prediction $\Yhat_t$ and state $U_t = \phi_t(\Fcal_t)$.
    \FOR{each $k \in \Kcal$}
        \STATE Find $k$-NN neighborhood $\Ncal_{t,k}(U_t) \subset \Dcal_{\mathrm{cal}}$ and radius $h_{t,k}(U_t)$.
        \STATE Compute $\widehat{q}_{t,k}(U_t) \leftarrow \frac{1}{k}\sum_{i \in \Ncal_{t,k}(U_t)} Z_{i,t}$.
        \STATE Set $L_{t,k}(U_t) \leftarrow \widehat{q}_{t,k}(U_t) - \btk(U_t) - \sqrt{\log(T|\Kcal|/\delta)/(2k)}$.
    \ENDFOR
    \STATE Set $L_t(U_t) \leftarrow \max_{k \in \Kcal} L_{t,k}(U_t)$.
    \IF{$L_t(U_t) \geq 1 - \alpha$}
        \RETURN $(\textsc{act}, \Yhat_t)$
    \ENDIF
\ENDFOR
\RETURN $\textsc{defer}$
\end{algorithmic}
\end{algorithm}

\begin{figure}[h]
\centering
\includegraphics[width=0.95\textwidth]{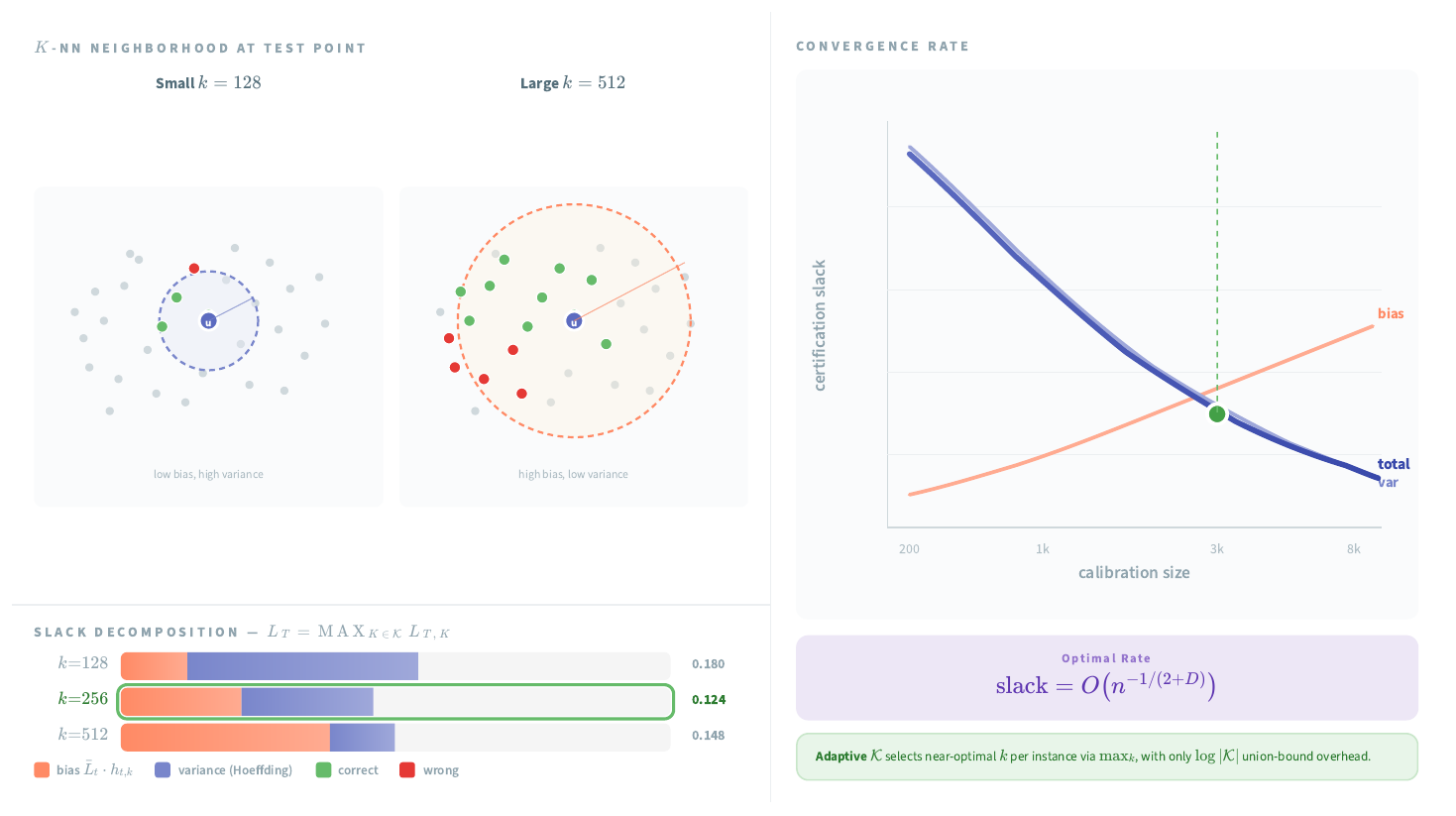}
\caption{\textbf{Bias--variance trade-off underlying the local lower bound.} \textit{Left}: small $k$ yields a tight neighborhood radius but high sampling variance; large $k$ reduces variance but inflates the bias envelope. The maximum over $\Kcal$ (Eq.~\eqref{eq:adaptive_lower}) lets each instance select its best $k$. \textit{Right}: certification slack $q_t(U_t) - L_t(U_t)$ vs.\ calibration size $n$ (fixed $D$), consistent with the nonparametric rate of Corollary~\ref{cor:rate}.}
\label{fig:bias_variance}
\end{figure}

\begin{figure}[h]
\centering
\includegraphics[width=0.95\textwidth]{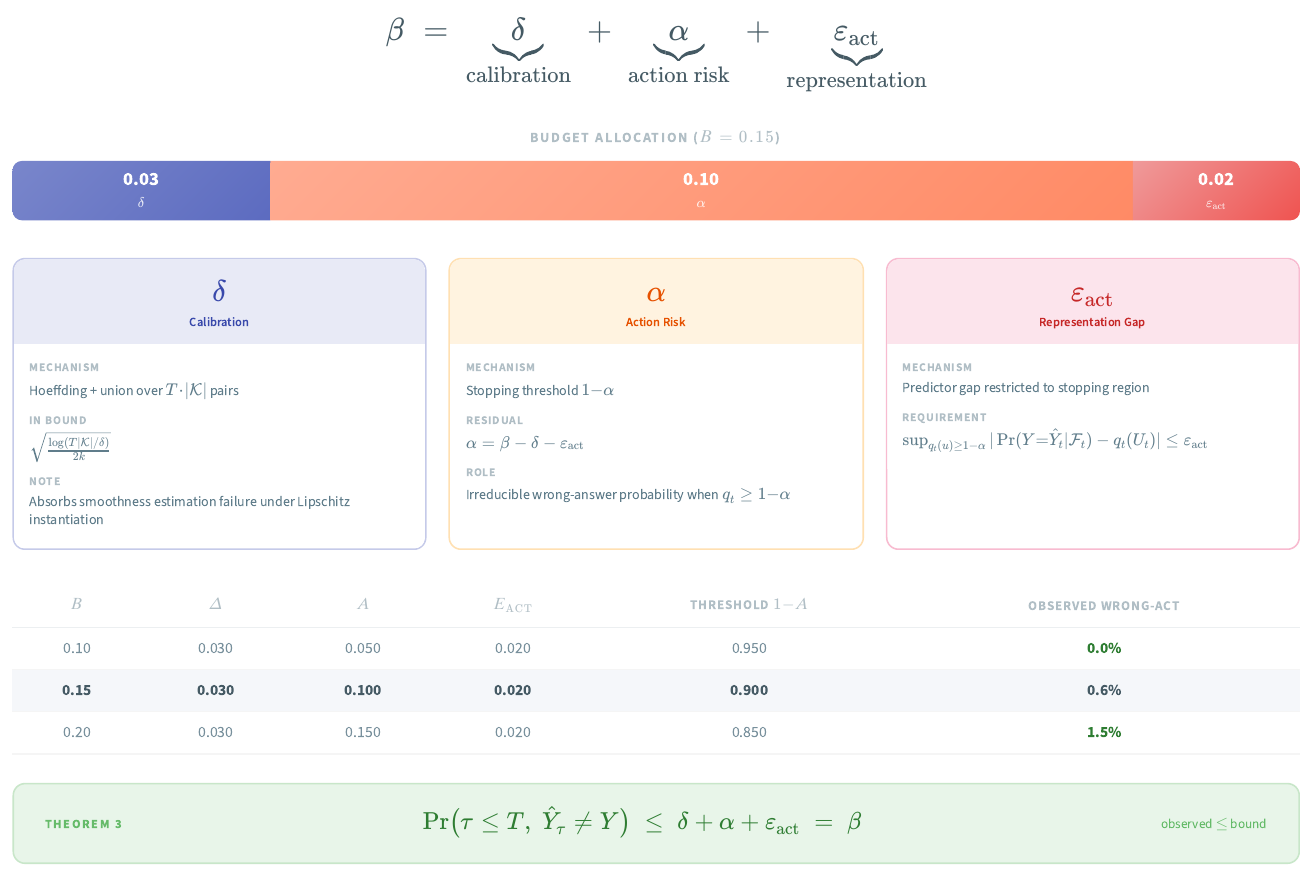}
\caption{\textbf{Budget-aligned wrong-action control.} The risk budget $\beta$ decomposes into three independently controlled terms. \textit{Top}: example allocation with $\beta = 0.15$. \textit{Middle}: mechanism for each error source. \textit{Bottom}: at all tested $\beta$ levels, observed WA stays below the bound.}
\label{fig:budget}
\end{figure}

\subsection{Debate protocol and social probability}
Debates use a structured prompt with \texttt{<reasoning>} and \texttt{<answer>} tags; at round $t > 1$, agents see all round-$(t{-}1)$ responses. Social probability uses uniform weights $w_i = 1/N$ over the $N$ agents: $\Psocial^{(t)}(y \mid x) = \frac{1}{N}\sum_{i=1}^N \mathbf{1}[\text{agent } i \text{ answers } y \text{ at round } t]$. The prediction is $\Yhat_t = \arg\max_y \Psocial^{(t)}(y \mid x)$, with ties broken by random selection.

\subsection{State normalization}
The default state is $U_t = (p_t^{(1)}, \Delta_t) \in \mathbb{R}^2$, where $p_t^{(1)} = \max_y \Psocial^{(t)}(y)$ and $\Delta_t = p_t^{(1)} - p_t^{(2)}$. Both components are normalized to $[0,1]$ via calibration quantile transform: for each coordinate $j \in \{1,2\}$, we set $U_t^{(j,\mathrm{norm})} = \widehat{F}_j(U_t^{(j)})$, where $\widehat{F}_j$ is the empirical CDF computed on the training split. Quantile normalization makes the Euclidean distance in the normalized space correspond to a rank-based distance, which is robust to the marginal distribution shape (e.g., heavy concentration of $p_t^{(1)}$ near $1/C$ or $1$). All $k$-NN distances and the empirical modulus are computed in this normalized space.

\subsection{$k$-NN search and bias envelope}
$k$-NN search is exact via brute-force (Euclidean distance in $\mathbb{R}^2$), which is fast at our scale ($n \leq 4{,}339$ per split, $D = 2$). We use the empirical modulus instantiation of the bias envelope: $\btk(u) = \omegahat_t(h_{t,k}(u))$, with $\omegahat_t$ estimated on $\Dcal_{\mathrm{mod}}$ (${\approx}\,830$ instances for MMLU-Pro) via~\eqref{eq:modulus}; see Appendix~\ref{app:smoothers} for smoother and modulus computation details.

\subsection{Budget allocation}
Default budget split: $\delta = 0.03$, $\epsact = 0.02$, giving $\alpha = \beta - 0.05$. When $\beta \leq \delta + \epsact + 0.001 = 0.051$, the method cannot certify (negative $\alpha$) and defers on all instances. Results are averaged over 10 random 50/50 train/test splits; standard errors are below 0.5\% for wrong-action rates.

\subsection{Training-selected $\lambda^\star$ procedure}
\label{app:lambda_selection}
The risk budget $\beta_d = \lambda^\star_d \cdot \widehat{e}_{T,d}^{\,\mathrm{train}}$ (\S\ref{sec:experiments}, eq.~\eqref{eq:beta_rule}) is selected as follows. Each 50/50 train/test split is further sub-split: the training half is divided 50/50 into $\Dcal_{\mathrm{train}_1}$ and $\Dcal_{\mathrm{train}_2}$. For each candidate $\lambda \in \{0.25, 0.50, \ldots, 5.00\}$:
\begin{enumerate}[nosep,leftmargin=1.5em]
    \item Compute the final-round error on $\Dcal_{\mathrm{train}_1}$: $\hat{e}_{T}^{(1)} = 1 - \frac{1}{|\Dcal_{\mathrm{train}_1}|}\sum_{i \in \Dcal_{\mathrm{train}_1}} Z_{i,T}$.
    \item Set $\beta = \lambda \cdot \hat{e}_{T}^{(1)}$. If $\beta > 0.99$, skip.
    \item Run our method with budget $\beta$ on $\Dcal_{\mathrm{train}_1}$ (calibration) $\to$ $\Dcal_{\mathrm{train}_2}$ (evaluation).
    \item Record automation rate $\widehat{\mathrm{Act}}(\lambda)$ and wrong-action ratio $\widehat{\mathrm{WA}}(\lambda)/\beta$.
\end{enumerate}
Select $\lambda^\star = \arg\max\{\widehat{\mathrm{Act}}(\lambda) : \widehat{\mathrm{WA}}(\lambda)/\beta \leq 0.10,\; \widehat{\mathrm{Act}}(\lambda) > 0\}$. If no $\lambda$ satisfies the constraint, the method defers on all test instances. This is standard hyperparameter selection on held-out training data; the formal guarantee (Theorem~\ref{thm:wrong_action}) holds for any fixed $\beta$, so selecting $\beta$ on training data does not compromise test-time validity.

\subsection{Computational complexity}
\label{app:complexity}
Let $n$ be the calibration size, $T$ the number of debate rounds, $|\Kcal|$ the number of candidate $k$ values, and $D$ the state dimension.

\textbf{Precomputation (one-time per split):} Running all $T$-round debates for $n$ calibration instances costs $O(nT)$ LLM calls. Computing the empirical modulus requires $O(|\Dcal_{\mathrm{mod}}|^2)$ pairwise distances on $\Dcal_{\mathrm{mod}}$ (the 20\% modulus split); with $|\Dcal_{\mathrm{mod}}| \approx n/5$, this is $O(n^2/25)$. Quantile normalization is $O(n \log n)$ per coordinate.

\textbf{Per-test-instance:} For each test instance, the cost is $O(T \cdot |\Kcal| \cdot n \cdot D)$ for brute-force $k$-NN search across all rounds and $k$ values. With $T = 4$, $|\Kcal| = 3$, $n \approx 3{,}300$, $D = 2$, this amounts to ${\sim}79{,}200$ distance computations---negligible compared to the LLM inference cost. Ball-tree or KD-tree indexing is unnecessary at this scale.

\section{Smoother Details and Modulus Estimation}
\label{app:smoothers}

The empirical modulus $\omegahat_t$ in~\eqref{eq:modulus} depends on the choice of pilot smoother $\qtilde_t$. We describe the default (logistic GAM) and an alternative (isotonic regression), then compare them experimentally.

\subsection{Default: logistic GAM}

A logistic generalized additive model (GAM) models the correctness probability as
\begin{equation}
    \log\frac{\qtilde_t(u)}{1 - \qtilde_t(u)} = f_1(p_t^{(1)}) + f_2(\Delta_t),
\end{equation}
where $f_1, f_2$ are smooth spline functions estimated via penalized maximum likelihood on $\Dcal_{\mathrm{mod}}$. The spline penalty controls overfitting, and the logistic link guarantees $\qtilde_t \in (0,1)$. We use cubic B-splines with 8 knots per feature and $\ell_2$ regularization ($C = 1.0$). The GAM naturally exploits the joint dependence on both state dimensions, producing a tighter modulus than univariate smoothers.

Given $\qtilde_t$, the empirical modulus is computed on the full 2D state space:
\begin{equation}
    \omegahat_t(r) = \max_{\substack{i,j \in \Dcal_{\mathrm{mod}} \\ d(U_{i,t}, U_{j,t}) \leq r}} |\qtilde_t(U_{i,t}) - \qtilde_t(U_{j,t})|.
\end{equation}
In practice, we tabulate $\omegahat_t$ on a grid of 200 radius values from $0$ to the 99th percentile of pairwise distances, enforce monotonicity (the modulus must be non-decreasing), and use linear interpolation for intermediate radii.

\subsection{Alternative: isotonic regression}

Isotonic regression fits $\qtilde_t$ as the best $L^2$ monotone approximation to $Z_t$:
\begin{equation}
    \qtilde_t = \arg\min_{f \text{ non-decreasing}} \sum_{i \in \Dcal_{\mathrm{mod}}} (Z_{i,t} - f(p_{i,t}^{(1)}))^2.
\end{equation}
This yields a monotone step function $\qtilde_t : [0,1] \to [0,1]$. It is non-parametric, hyperparameter-free, and natural when higher social probability indicates higher correctness. However, it is univariate (uses only $p_t^{(1)}$, ignoring $\Delta_t$) and cannot capture non-monotone relationships.

\subsection{Experimental comparison}

Table~\ref{tab:smoother_comparison} compares the two smoothers on pooled MMLU-Pro. Both achieve valid coverage (WA $\leq \beta = 0.15$). The logistic GAM produces a tighter modulus (0.122 vs.\ 0.285 at $r = 0.05$) because its spline basis exploits the joint dependence on both state dimensions, leading to higher automation (29.8\% vs.\ 21.8\%). We use the logistic GAM as the default throughout the paper.

\begin{table}[h]
\caption{Smoother comparison on pooled MMLU-Pro ($n\!=\!8312$, $\beta = 0.15$). Both smoothers produce valid bounds; the logistic GAM achieves a tighter modulus and higher automation.}
\label{tab:smoother_comparison}
\centering
\small
\begin{tabular}{lccccc}
\toprule
\textbf{Smoother} & \textbf{WA (\%)} & \textbf{Auto.\ (\%)} & \textbf{Avg Round} & $\omegahat_1(0.05)$ & $\omegahat_1(0.10)$ \\
\midrule
Logistic GAM (default) & 0.6 & 29.8 & 3.54 & 0.122 & 0.238 \\
Isotonic regression & 0.4 & 21.8 & 3.78 & 0.285 & 0.372 \\
\bottomrule
\end{tabular}
\end{table}

\subsection{Formal analysis: when does $\omegahat_t \geq \omega_t$?}
\label{app:modulus_analysis}

The validity of the empirical modulus envelope rests on $\omegahat_t(r) \geq \omega_t(r)$ for all relevant radii $r$. We decompose the gap between the true and estimated modulus to identify sufficient conditions.

Let $\omega_t(r) = \sup_{d(u,v) \leq r} |q_t(u) - q_t(v)|$ be the true modulus of $q_t$, and $\omega_{\tilde{q}}(r) = \sup_{d(u,v) \leq r} |\qtilde_t(u) - \qtilde_t(v)|$ the modulus of the pilot smoother. The empirical modulus $\omegahat_t(r)$ approximates $\omega_{\tilde{q}}(r)$ via a finite maximum over $\Dcal_{\mathrm{mod}}$ points. The error decomposes as:
\begin{equation}
\label{eq:modulus_decomp}
    \omega_t(r) - \omegahat_t(r) \leq \underbrace{[\omega_t(r) - \omega_{\tilde{q}}(r)]}_{\text{smoother error}} + \underbrace{[\omega_{\tilde{q}}(r) - \omegahat_t(r)]}_{\text{discretization error}}.
\end{equation}

\textbf{Smoother error.} By the triangle inequality on the supremum:
\begin{equation}
    |\omega_t(r) - \omega_{\tilde{q}}(r)| \leq 2\,\|q_t - \qtilde_t\|_\infty,
\end{equation}
where $\|q_t - \qtilde_t\|_\infty = \sup_{u \in \Ucal} |q_t(u) - \qtilde_t(u)|$ is the uniform approximation error of the pilot smoother. If $\qtilde_t$ is a consistent estimator of $q_t$ (e.g., a logistic GAM with penalty selected by BIC), then $\|q_t - \qtilde_t\|_\infty \to 0$ as $|\Dcal_{\mathrm{mod}}| \to \infty$.

\textbf{Discretization error.} The modulus $\omega_{\tilde{q}}(r)$ is a supremum over all pairs $(u,v)$ with $d(u,v) \leq r$, while $\omegahat_t(r)$ maximizes only over pairs in $\Dcal_{\mathrm{mod}}$. If $\qtilde_t$ is $L_{\tilde{q}}$-Lipschitz (which holds for logistic GAMs with bounded coefficients), and $\Dcal_{\mathrm{mod}}$ has covering number $\Ncal(\epsilon)$ at scale $\epsilon$, then $\omega_{\tilde{q}}(r) - \omegahat_t(r) \leq 2L_{\tilde{q}} \cdot \epsilon_{\mathrm{cover}}$, where $\epsilon_{\mathrm{cover}}$ is the covering radius. For $|\Dcal_{\mathrm{mod}}| = m$ points in $\mathbb{R}^D$, $\epsilon_{\mathrm{cover}} = O(m^{-1/D})$.

\textbf{Sufficient condition for validity.} Combining, $\omegahat_t(r) \geq \omega_t(r)$ holds when:
\begin{equation}
    2\,\|q_t - \qtilde_t\|_\infty + 2L_{\tilde{q}} \cdot O(m^{-1/D}) \leq 0,
\end{equation}
which is not guaranteed in finite samples. In practice, two factors provide a safety margin: (i)~penalized splines \emph{under}-approximate variation, making $\omega_{\tilde{q}}(r) \leq \omega_t(r)$ unlikely but $\qtilde_t$ smoother than $q_t$; (ii)~when $\qtilde_t$ under-smooths (tracks noise), $\omega_{\tilde{q}}(r) > \omega_t(r)$, \emph{inflating} the envelope conservatively. The stress test (Table~\ref{tab:envelope_stress}) validates this empirically: even $2\times$ inflation of $\omegahat_t$ leaves WA unchanged, confirming that the estimated modulus is already above $\omega_t$.

\paragraph{Impact on validity.}
The theoretical guarantee (Corollary~\ref{cor:modulus}) holds for any smoother, provided $\omegahat_t(r) \geq \omega_t(r)$ with high probability. The choice of smoother affects the \emph{tightness} of the bound, not its validity. An undersmoother (too flexible) will track noise in $Z$, inflating $\omegahat_t$ and making the bias envelope conservative. An oversmoother (too rigid) will underestimate the true variation of $q_t$, potentially violating the envelope. The logistic GAM strikes a good balance: the spline penalty prevents overfitting while allowing multivariate dependence, producing a tighter modulus than univariate isotonic regression.


\subsection{Envelope stress test}
\label{app:envelope_stress}

A central concern is whether $\omegahat_t \geq \omega_t$ holds in practice---i.e., whether the empirical modulus overestimates the true modulus of continuity. We test robustness by inflating $\omegahat_t$ by multiplicative factors $\{1.0, 1.25, 1.5, 2.0\}$ and measuring the impact on Act and WA (Table~\ref{tab:envelope_stress}). If the baseline modulus were a tight upper bound on $\omega_t$, inflating it would shrink the bias correction $\btk(u) = \omegahat_t(h_{t,k}(u))$, lowering $L_t$ and reducing automation. The results show that even at $2\times$ inflation, Act degrades only modestly (MMLU-Pro: 29.8\%$\to$27.8\%; BBH: 74.9\%$\to$72.3\%) while WA remains unchanged. This insensitivity confirms that the empirical modulus is already conservative (substantially above $\omega_t$), providing a safety margin against modulus misspecification. ARC remains at 0\% Act at all factors because it does not certify at $\beta = 0.15$.

\begin{table}[h]
\caption{\textbf{Envelope stress test.} Inflating $\omegahat_t$ by a multiplicative factor ($\beta=0.15$). Higher inflation simulates modulus underestimation, tightening the lower bound. WA stays well below $\beta$ at all inflation levels; Act degrades gracefully.}
\label{tab:envelope_stress}
\centering\small\renewcommand{\arraystretch}{1.08}
\begin{tabular}{@{} l rr rr rr rr @{}}
\toprule
 & \multicolumn{2}{c}{$1.0\times$} & \multicolumn{2}{c}{$1.25\times$} & \multicolumn{2}{c}{$1.5\times$} & \multicolumn{2}{c}{$2.0\times$} \\
\cmidrule(lr){2-3}\cmidrule(lr){4-5}\cmidrule(lr){6-7}\cmidrule(lr){8-9}
\textbf{Dataset} & \textbf{WA} & \textbf{Act} & \textbf{WA} & \textbf{Act} & \textbf{WA} & \textbf{Act} & \textbf{WA} & \textbf{Act} \\
\midrule
MMLU-Pro & 0.6 & 29.8 & 0.6 & 28.6 & 0.6 & 28.2 & 0.6 & 27.8 \\
ARC & 0.0 & 0.0 & 0.0 & 0.0 & 0.0 & 0.0 & 0.0 & 0.0 \\
BBH & 1.1 & 74.9 & 1.1 & 74.0 & 1.0 & 73.2 & 1.0 & 72.3 \\
\bottomrule
\end{tabular}
\renewcommand{\arraystretch}{1.0}
\end{table}


\section{Extended Related Work}
\label{app:extended_related}

We place our framework in the context of six adjacent literatures: (i) multi-agent LLM debate and inference-time compute; (ii) LLM confidence, calibration, and selective generation; (iii) conformal prediction, risk control, and anytime-valid inference; (iv) local and conditional conformal methods; (v) distribution-free inference and concentration; and (vi) classical foundations ($k$-NN theory, selective classification, opinion dynamics).

\paragraph{Multi-agent LLM debate and inference-time compute.}
Multi-agent debate \citep{du2023improving, liang2024encouraging} and scalable multi-agent reasoning \citep{qian2024scaling} improve accuracy by having LLMs critique each other's answers; ChatEval \citep{chan2024chateval} and debate-for-judgment pipelines \citep{khan2024debating} extend this to evaluation. Theoretical motivations include AI safety via debate \citep{irving2018ai}. Recent work questions whether more debate rounds always help \citep{choi2025debate, yi2025debate, wang2024debate, li2024more, zhang2024exploring} and studies stability of collective answers \citep{hu2025multi}. A separate line improves single-agent reasoning via self-consistency \citep{wang2023selfconsistency} and chain-of-thought \citep{wei2022chain}. None of these works provide instance-level wrong-action certificates or a budgeted stopping rule; they treat debate as an accuracy-improving heuristic rather than a certifiable procedure.

\paragraph{LLM confidence, calibration, and selective generation.}
LLMs can self-assess correctness via probe-based methods \citep{kadavath2022language}, verbalized confidence \citep{tian2023just, xiong2024can, lin2022teaching}, or consistency checks \citep{manakul2023selfcheckgpt, kuhn2023semantic}. Broader surveys cover estimation strategies \citep{geng2024survey} and calibration \citep{guo2017calibration, naeini2015obtaining, platt1999probabilistic, sharma2024towards, marx2024calibration, vaicenavicius2019evaluating, blasiok2023unifying}; see also verified-uncertainty calibration \citep{kumar2019verified} and context-dependent effects \citep{zhao2021calibrating}. Selective generation \citep{ren2023robots, varshney2023stitch, feng2024dont, yang2024llm} refuses to answer below a confidence threshold. All these methods operate on marginal (instance-averaged) scores; our state-conditional lower bound instead certifies the correctness probability at a specific debate state, yielding an explicit wrong-action budget rather than a tuned score threshold.

\paragraph{Conformal prediction and risk control.}
Split and full conformal prediction \citep{vovk2005algorithmic, papadopoulos2008inductive, lei2013distribution, lei2014distribution, angelopoulos2023conformal} provide marginal coverage for prediction sets; extensions to regression and classification include \citep{romano2019conformalized, romano2020classification, angelopoulos2021uncertainty}. Conformal risk control generalizes coverage to arbitrary risk functions \citep{bates2021distribution, angelopoulos2024conformal, angelopoulos2022learn, fisch2022conformal}, and language-model applications \citep{kumar2023conformal, quach2024conformal} adapt these tools to generation. Adaptive and online conformal methods track distribution shift over time \citep{gibbs2021adaptive, gibbs2024conformal, barber2023conformal}, while game-theoretic and e-value formulations unify anytime-valid inference \citep{shafer2011test, ramdas2022admissible, ramdas2023game, vovk2021evalues, grunwald2020safe, howard2021time, waudbysmith2024estimating, johari2022always, ville1939etude, wald1945sequential}. Our framework differs in target and geometry: instead of constructing marginally valid prediction sets, we build a \emph{pointwise} lower bound on state-conditional correctness and use it as a stopping rule; conformal risk control is a natural aggregate baseline (Table~\ref{tab:main}).

\paragraph{Local and conditional conformal methods.}
Classical coverage guarantees are marginal; strengthening them to condition on features is known to be fundamentally hard \citep{vovk2012conditional, barber2021limits}. Localized conformal methods \citep{guan2023localized, hore2023conformal, gupta2020distribution} reweight calibration data near a query point, and Mondrian / group-conditional conformal \citep{vovk2005algorithmic, barber2023conformal} partitions the feature space. Our $k$-NN bound is similar in spirit---local reweighting via nearest neighbors---but targets a different functional (wrong-action probability in the action region, not coverage of a prediction set), and pairs the local estimator with an explicit bias envelope $\btk$ rather than assuming smoothness implicitly.

\paragraph{Distribution-free inference and concentration.}
The concentration inequalities that underpin our simultaneous validity proof are classical: Hoeffding's inequality \citep{hoeffding1963probability}, empirical Bernstein bounds \citep{maurer2009empirical}, and general concentration machinery \citep{boucheron2013concentration, wasserman2006all}. Stone's theorem \citep{stone1977consistent} and its $k$-NN specializations \citep{cover1967nearest, fix1951discriminatory, devroye1996probabilistic, gyorfi2002distribution, biau2015lectures, chaudhuri2010rates} give the nonparametric rate underlying Corollary~\ref{cor:rate}; Tsybakov \citep{tsybakov2009introduction} provides the minimax lower bounds.

\paragraph{Selective classification and opinion dynamics.}
Selective classification \citep{elyaniv2010foundations, geifman2017selective} provides accuracy--coverage trade-offs under a reject option; our act-or-defer rule is a multi-round, state-conditional, budget-aligned refinement. The ``social probability'' $\Psocial^{(t)}$ connects to classical opinion-dynamics models---DeGroot consensus \citep{degroot1974reaching}, bounded-confidence dynamics \citep{hegselmann2002opinion, deffuant2000mixing}, and their steady-state analyses \citep{baccelli2021steady}---and to foundational aggregation theory \citep{arrow1951social}. We treat agent debate as a signal-generating process rather than a truth-seeking dynamical system: the state $U_t$ summarizes the resulting signal, and calibration data determines whether the signal warrants action.

\paragraph{Benchmarks and auxiliary.}
We use standard LLM reasoning benchmarks \citep{hendrycks2021measuring, wang2024mmlu, clark2018think, suzgun2023challenging, liu2020logiqa, sprague2024musr, rein2023gpqa} throughout.

Compared to all of the above, the distinguishing combination in this paper is: (a) a \emph{state-conditional} local lower bound, (b) a \emph{budget-aligned} stopping rule with an explicit three-way decomposition $\beta = \delta + \alpha + \epsact$, (c) \emph{adaptive} $k$ with an oracle-style slack inequality (Theorem~\ref{thm:oracle_slack}), and (d) application to multi-round multi-agent debate with diagnostic proxies for each structural assumption. We are not aware of prior work combining these four elements in a single framework.

\section{Methodological Comparison}
\label{app:method_comparison}

\begin{table}[h]
\caption{\textbf{Methodological comparison.}
\,\cmark\,=\,full support;\;
\,\hmark\,=\,partial;\;
\,\xmark\,=\,not supported.}
\label{tab:method_comparison}
\centering\small\renewcommand{\arraystretch}{1.35}
\setlength{\tabcolsep}{5pt}
\begin{tabular}{@{} l ccccc @{}}
\toprule
\textbf{Method} & \rotatebox{60}{\textbf{Adaptive stop}} & \rotatebox{60}{\textbf{Local / state-cond.}} & \rotatebox{60}{\textbf{WA budget}} & \rotatebox{60}{\textbf{Assump.\ transparent}} & \rotatebox{60}{\textbf{Multi-agent}} \\
\midrule
Consensus thr. & \cmark & \xmark & \xmark & \xmark & \cmark \\
Selective pred. & \hmark & \xmark & \hmark & \xmark & \hmark \\
CRC & \xmark & \xmark & \cmark & \cmark & \hmark \\
Local conformal & \hmark & \cmark & \xmark & \hmark & \xmark \\
\rowcolor{blue!6}
\textbf{Ours} & \cmark & \cmark & \cmark & \cmark & \cmark \\
\bottomrule
\end{tabular}
\renewcommand{\arraystretch}{1.0}
\end{table}

\section{Full Main Results}
\label{app:full_table1}

Table~\ref{tab:main_full} reports all nine methods on all six datasets, extending the summary in Table~\ref{tab:main} with additional baselines including Selective Prediction (risk-controlled confidence threshold tuned on calibration data) and CRC (conformal risk control using global score quantile).

\begin{table*}[h]
\caption{Full safety--efficiency results at $\beta = 0.15$ ($T\!=\!4$), extending Table~\ref{tab:main} with all nine baselines from Section~\ref{sec:experiments}. $\dagger$\,marks WA\,$>\beta$. Act\,=\,automation (\%), Acc\,=\,Acc$|$Act (\%), WA\,=\,wrong-action (\%). \emph{Risk-control}: CRC~\citep{angelopoulos2024conformal}; Selective Prediction~\citep{elyaniv2010foundations, geifman2017selective}; Calibrated Learned; Isotonic Confidence. \emph{Ablations of our method}: $k$NN no bias (no modulus); Final-round (bound at $t\!=\!T$). \emph{Heuristics}: Consensus; Confidence Threshold; Learned Stopper. All tuned on training split, evaluated on held-out test split.}
\label{tab:main_full}
\centering\small\renewcommand{\arraystretch}{1.08}
\setlength{\tabcolsep}{2.8pt}
\begin{tabular}{@{} l rrr rrr rrr rrr @{}}
\toprule
& \multicolumn{3}{c}{\textbf{MMLU-Pro} {\scriptsize($n\!=\!8312$)}}
& \multicolumn{3}{c}{\textbf{ARC} {\scriptsize($n\!=\!2590$)}}
& \multicolumn{3}{c}{\textbf{BBH} {\scriptsize($n\!=\!2395$)}}
& \multicolumn{3}{c}{\textbf{MuSR} {\scriptsize($n\!=\!756$)}} \\
\cmidrule(lr){2-4}\cmidrule(lr){5-7}\cmidrule(lr){8-10}\cmidrule(lr){11-13}
\textbf{Method} & Act & Acc & WA & Act & Acc & WA & Act & Acc & WA & Act & Acc & WA \\
\midrule
\rowcolor{blue!6}
\textbf{Ours} & \textbf{29.8} & \textbf{98.0} & \textbf{0.6} & 0.0 & --- & 0.0 & \textbf{74.9} & \textbf{98.5} & \textbf{1.1} & 0.0 & --- & 0.0 \\
\midrule
CRC & 97.6 & 84.8 & 14.8$\dagger$ & 99.9 & 95.8 & 4.2 & 100 & 91.6 & 8.4 & 43.3 & 67.2 & 14.2 \\
Selective Prediction & 97.3 & 85.1 & 14.5 & 100 & 95.7 & 4.3 & 99.9 & 91.7 & 8.2 & 43.4 & 67.2 & 14.2 \\
Calibrated Learned & 95.9 & 85.4 & 14.0 & 100 & 95.7 & 4.3 & 99.7 & 91.4 & 8.6 & 18.9 & 46.8 & 6.0 \\
Isotonic Confidence & 97.3 & 85.0 & 14.6 & 100 & 95.7 & 4.3 & 99.7 & 91.7 & 8.3 & 7.6 & 30.0 & 3.0 \\
\midrule
$k$NN no bias & 37.9 & 97.7 & 0.9 & 32.6 & 97.6 & 0.9 & 79.1 & 97.9 & 1.6 & 0.0 & --- & 0.0 \\
Final-round & 29.3 & 97.9 & 0.6 & 0.2 & 9.6 & 0.0 & 74.7 & 98.3 & 1.3 & 0.0 & --- & 0.0 \\
\midrule
Consensus & 94.5 & 84.5 & 15.5$\dagger$ & 99.0 & 95.7 & 4.3 & 98.2 & 91.5 & 8.5 & 91.9 & 65.7 & 34.3$\dagger$ \\
Confidence Threshold & 69.8 & 84.5 & 15.5$\dagger$ & 82.8 & 95.7 & 4.3 & 87.1 & 91.5 & 8.5 & 32.9 & 66.0 & 34.0$\dagger$ \\
Learned Stopper & 96.6 & 83.6 & 16.4$\dagger$ & 100 & 95.7 & 4.3 & 99.4 & 91.7 & 8.3 & 98.8 & 63.5 & 36.5$\dagger$ \\
\midrule
\color{gray}\textit{Oracle} & \color{gray}\textit{86.0} & \color{gray}\textit{100} & \color{gray}\textit{0.0} & \color{gray}\textit{96.3} & \color{gray}\textit{100} & \color{gray}\textit{0.0} & \color{gray}\textit{92.6} & \color{gray}\textit{100} & \color{gray}\textit{0.0} & \color{gray}\textit{68.9} & \color{gray}\textit{100} & \color{gray}\textit{0.0} \\
\bottomrule
\end{tabular}
\renewcommand{\arraystretch}{1.0}
\end{table*}

\section{Budget Sensitivity, Ablations, and Robustness}
\label{app:budget_ablations}
\label{app:budget_sweep}
\label{app:full_baselines}

This section examines how our method behaves as the risk budget $\beta$ varies, verifies that each design choice contributes meaningfully, and confirms robustness to fixed-budget settings. We organize the analysis into four parts: (\S\ref{app:budget_frontier})~the risk--automation frontier under error-normalized budgets, (\S\ref{app:budget_alignment})~empirical budget alignment showing WA $\leq \beta$ holds uniformly, (\S\ref{app:ablations})~ablations isolating the contributions of adaptive $k$ and calibration size, and (\S\ref{app:fixed_beta})~a fixed-budget robustness check applying the same $\beta$ across all datasets.

\subsection{Risk--automation frontier}
\label{app:budget_frontier}

Table~\ref{tab:main} reports results under training-selected $\lambda^\star$, which determines a single $\beta$ per dataset. Here we trace the full frontier by sweeping $\rho = \beta / e_T$ (the error-normalized multiplier) across $\{1.0, 1.5, 2.0, 2.5, 3.0\}$. This reveals how automation and accuracy trade off as the allowed risk budget grows.

\begin{table*}[h]
\caption{\textbf{Risk--automation frontier} (Ours, $T\!=\!4$).
$\rho = \beta / e_T$ is the error-normalized multiplier; $\beta = \rho \times e_T$ per dataset.
GPQA and MuSR omitted (0\% Act at $\beta \leq 0.35$).
ARC uses absolute $\beta$ because its low $e_T = 0.045$ makes $\rho$ uninformative.
WA/$\beta$ stays below 0.26 at all operating points.}
\label{tab:ours_frontier}
\centering\small\renewcommand{\arraystretch}{1.08}
\setlength{\tabcolsep}{3pt}
\begin{tabular}{@{} l r rrrr rrrr rrrr rrrr @{}}
\toprule
& & \multicolumn{4}{c}{\textbf{MMLU-Pro}} & \multicolumn{4}{c}{\textbf{BBH}} & \multicolumn{4}{c}{\textbf{ARC}} & \multicolumn{4}{c}{\textbf{LogiQA}} \\
\cmidrule(lr){3-6}\cmidrule(lr){7-10}\cmidrule(lr){11-14}\cmidrule(lr){15-18}
$\boldsymbol{\rho}$ & $\boldsymbol{\beta}$ & Act & Acc & WA & {\scriptsize WA/$\beta$} & Act & Acc & WA & {\scriptsize WA/$\beta$} & Act & Acc & WA & {\scriptsize WA/$\beta$} & Act & Acc & WA & {\scriptsize WA/$\beta$} \\
\midrule
1.0 & \textsuperscript{*} & 32.8 & 97.9 & 0.7 & .04 & --- & --- & --- & --- & --- & --- & --- & --- & 7.2 & 83.5 & 0.6 & .02 \\
1.5 & \textsuperscript{*} & 54.9 & 96.3 & 2.1 & .09 & --- & --- & --- & --- & --- & --- & --- & --- & 53.7 & 86.9 & 7.0 & .19 \\
2.0 & \textsuperscript{*} & 65.6 & 95.1 & 3.2 & .10 & 77.3 & 98.1 & 1.5 & .09 & 74.4 & 96.9 & 2.3 & .12 & 75.0 & 83.0 & 12.8 & .26 \\
2.5 & \textsuperscript{*} & --- & --- & --- & --- & 80.9 & 97.4 & 2.1 & .10 & 92.5 & 96.5 & 3.2 & .13 & --- & --- & --- & --- \\
3.0 & \textsuperscript{*} & --- & --- & --- & --- & 82.5 & 96.9 & 2.5 & .10 & 97.3 & 96.1 & 3.9 & .13 & --- & --- & --- & --- \\
\bottomrule
\multicolumn{18}{@{}l}{\scriptsize\textsuperscript{*}$\beta = \rho \times e_T$ per dataset ($e_T$: .157, .083, .045, .242). ARC rows use absolute $\beta$\,=\,.20, .25, .30.}
\end{tabular}
\renewcommand{\arraystretch}{1.0}
\end{table*}

Several patterns emerge. First, automation increases monotonically with $\rho$: on MMLU-Pro, Act grows from 32.8\% ($\rho\!=\!1.0$) to 65.6\% ($\rho\!=\!2.0$), while Acc$|$Act decreases only modestly (97.9\%$\to$95.1\%). Second, WA/$\beta$ remains far below 1.0 at every operating point---the method is using a small fraction of the allowed budget. Third, BBH and ARC do not activate at $\rho\!=\!1.0$ because $\beta = e_T$ implies $\alpha = e_T - 0.05$, which is too small for BBH ($\alpha = 0.033$) to certify any instances. The ``---'' entries reflect this: the method correctly defers when the budget is too tight for reliable certification. LogiQA shows the steepest Act--Acc trade-off at high $\rho$, consistent with its higher base error ($e_T = 0.242$) and larger representation gap (Table~\ref{tab:audit}).

\subsection{Budget alignment}
\label{app:budget_alignment}

A core claim of the framework is that observed WA stays below $\beta$ at every budget level, not just the training-selected one. Table~\ref{tab:budget} tests this by applying fixed $\beta \in \{0.10, 0.15, 0.20, 0.25, 0.30\}$ uniformly to all datasets.

\begin{table*}[h]
\caption{\textbf{Budget alignment} (Ours only; $\delta\!=\!0.03$, $\epsact\!=\!0.02$ throughout). WA stays below $\beta$ on all datasets at every level. ``0.0 / 0.0'' means the method defers entirely.}
\label{tab:budget}
\centering\small\renewcommand{\arraystretch}{1.10}
\begin{tabular}{@{} c c rr rr rr rr @{}}
\toprule
 & & \multicolumn{2}{c}{\textbf{MMLU-Pro}} & \multicolumn{2}{c}{\textbf{ARC}} & \multicolumn{2}{c}{\textbf{BBH}} & \multicolumn{2}{c}{\textbf{MuSR}} \\
\cmidrule(lr){3-4} \cmidrule(lr){5-6} \cmidrule(lr){7-8} \cmidrule(lr){9-10}
$\beta$ & $\alpha$ & WA & Act & WA & Act & WA & Act & WA & Act \\
\midrule
0.10 & 0.05 & 0.0 & 0.0 & 0.0 & 0.0 & 0.0 & 0.0 & 0.0 & 0.0 \\
0.15 & 0.10 & 0.6 & 29.8 & 0.0 & 0.0 & 1.1 & 74.9 & 0.0 & 0.0 \\
0.20 & 0.15 & 1.5 & 48.8 & 2.3 & 74.4 & 2.0 & 80.3 & 0.0 & 0.0 \\
0.25 & 0.20 & 2.2 & 57.0 & 3.2 & 92.5 & 2.5 & 82.6 & 0.0 & 0.0 \\
0.30 & 0.25 & 3.1 & 64.1 & 3.9 & 97.3 & 2.8 & 84.0 & 0.0 & 0.0 \\
\bottomrule
\end{tabular}
\renewcommand{\arraystretch}{1.0}
\end{table*}

The key observation is that WA $< \beta$ holds in every cell---no budget is ever exceeded. At $\beta = 0.10$, $\alpha = 0.05$ is too tight for any dataset to certify. As $\beta$ increases, datasets activate in order of difficulty: BBH first ($\beta = 0.15$), then MMLU-Pro ($\beta = 0.15$), ARC ($\beta = 0.20$), and LogiQA only at $\beta = 0.30$. MuSR never activates at $\beta \leq 0.30$ due to its high base error ($e_T = 0.344$) and small sample size ($n = 756$). This monotonic activation pattern is a direct consequence of the certification slack $q_t(U_t) - L_t(U_t) = O(n^{-1/(2+D)})$: datasets with higher $e_T$ or smaller $n$ have larger slack, requiring a larger $\beta$ before $L_t$ can cross the $1 - \alpha$ threshold.

\subsection{Ablations}
\label{app:ablations}

We isolate two design choices on pooled MMLU-Pro at $\beta = 0.15$: adaptive $k$-selection versus fixed $k$, and sensitivity to calibration size $n_{\mathrm{loc}}$.

\begin{table}[h]
\caption{Ablations on pooled MMLU-Pro ($\beta = 0.15$, $n\!=\!8312$). \textit{Left}: adaptive-$k$ vs.\ fixed-$k$. \textit{Right}: sensitivity to calibration size $n_{\mathrm{loc}}$.}
\label{tab:ablations}
\centering
\small
\begin{minipage}[t]{0.48\textwidth}
\centering
\setlength{\tabcolsep}{4pt}
\begin{tabular}{lcccc}
\toprule
\textbf{Method} & \textbf{WA} & \textbf{Act} & \textbf{Acc$|$Act} & \textbf{Bias} \\
\midrule
Fixed $k\!=\!128$ & 0.0 & 0.0 & --- & 0.054 \\
Fixed $k\!=\!256$ & 0.0 & 2.9 & 98.8 & 0.080 \\
Fixed $k\!=\!512$ & 0.7 & 33.2 & 97.9 & 0.120 \\
\midrule
\textbf{Adaptive $\Kcal$} & \textbf{0.6} & \textbf{29.8} & \textbf{98.0} & --- \\
\bottomrule
\end{tabular}
\end{minipage}
\hfill
\begin{minipage}[t]{0.48\textwidth}
\centering
\begin{tabular}{cccc}
\toprule
$n_{\mathrm{loc}}$ & \textbf{WA (\%)} & \textbf{Act (\%)} & \textbf{Med.\ $h_{t,k}$} \\
\midrule
1{,}329 & 0.0 & 0.1 & 0.045 \\
2{,}659 & 0.5 & 23.9 & 0.036 \\
5{,}319 & 0.6 & 26.8 & 0.029 \\
\bottomrule
\end{tabular}
\end{minipage}
\end{table}

\textbf{Adaptive vs.\ fixed $k$ (left).} Fixed $k\!=\!128$ has a tight bias envelope (median bias 0.054) but the Hoeffding term $\sqrt{\log(12/0.03)/256} \approx 0.15$ is large, preventing certification. Fixed $k\!=\!512$ reduces the Hoeffding term to 0.09 but inflates the bias to 0.120, overshooting the optimal trade-off. The adaptive family $\Kcal = \{128, 256, 512\}$ lets each instance select its own best $k$ via $L_t = \max_k L_{t,k}$, achieving comparable Act (29.8\%) to the best fixed $k$ (33.2\% at $k\!=\!512$) with lower WA (0.6\% vs.\ 0.7\%). The $\log|\Kcal|$ penalty for adaptivity is negligible ($\log 3 / \log 1 \approx 1.1\times$ increase in the concentration term).

\textbf{Calibration size sensitivity (right).} Reducing $n_{\mathrm{loc}}$ from 5,319 to 1,329 (4$\times$ reduction) collapses Act from 26.8\% to 0.1\%. This is because the $k$-NN radius grows as $h_{t,k} \propto (k/n)^{1/D}$: halving $n$ increases $h$ by $2^{1/2} \approx 1.4\times$, which inflates the bias envelope and pushes $L_t$ below threshold. The median $k$-NN radius confirms this: $h$ grows from 0.029 to 0.045 as $n$ shrinks. This rate sensitivity motivates pooling subtasks within benchmarks (e.g., all 10 BBH subtasks) to maximize effective $n$.

\subsection{Fixed-budget robustness}
\label{app:fixed_beta}

The training-selected $\lambda^\star$ determines a dataset-specific $\beta$. To verify that the method is not overfit to these particular operating points, Table~\ref{tab:fixed_beta} applies the same fixed $\beta$ to all six datasets simultaneously.

\begin{table}[h]
\caption{\textbf{Fixed-budget robustness} (Ours only, $T\!=\!4$).
Same $\beta$ applied to all six datasets. The method activates progressively as the budget is relaxed and defers on hard datasets at stringent budgets. ``---'' = 0\% automation (full deferral).}
\label{tab:fixed_beta}
\centering\small\renewcommand{\arraystretch}{1.08}
\setlength{\tabcolsep}{3pt}
\begin{tabular}{@{} l cccc cccc cccc @{}}
\toprule
& \multicolumn{4}{c}{$\beta = 0.15$} & \multicolumn{4}{c}{$\beta = 0.20$} & \multicolumn{4}{c}{$\beta = 0.30$} \\
\cmidrule(lr){2-5}\cmidrule(lr){6-9}\cmidrule(lr){10-13}
\textbf{Dataset} & \textbf{Act} & \textbf{Acc} & \textbf{WA} & \textbf{Rnd} & \textbf{Act} & \textbf{Acc} & \textbf{WA} & \textbf{Rnd} & \textbf{Act} & \textbf{Acc} & \textbf{WA} & \textbf{Rnd} \\
\midrule
MMLU-Pro & 29.8 & 98.0 & 0.6 & 3.5 & 48.8 & 96.8 & 1.6 & 3.0 & 64.1 & 95.2 & 3.1 & 2.5 \\
BBH & 74.9 & 98.5 & 1.1 & 2.5 & 80.3 & 97.5 & 2.0 & 1.9 & 84.0 & 96.7 & 2.8 & 1.8 \\
ARC & --- & --- & --- & --- & 74.4 & 96.9 & 2.3 & 2.5 & 97.3 & 96.1 & 3.9 & 1.2 \\
LogiQA & --- & --- & --- & --- & --- & --- & --- & --- & 22.2 & 89.5 & 2.4 & 3.7 \\
GPQA & --- & --- & --- & --- & --- & --- & --- & --- & --- & --- & --- & --- \\
MuSR & --- & --- & --- & --- & --- & --- & --- & --- & --- & --- & --- & --- \\
\bottomrule
\end{tabular}
\renewcommand{\arraystretch}{1.0}
\end{table}

The fixed-budget results are consistent with the frontier in Table~\ref{tab:ours_frontier}: BBH achieves 74.9\% Act at $\beta = 0.15$ (the most budget-efficient dataset); MMLU-Pro reaches 64.1\% at $\beta = 0.30$; ARC requires $\beta \geq 0.20$ to activate because its low $e_T = 0.045$ means the action threshold $1 - \alpha = 1 - (\beta - 0.05)$ is close to 1.0 at small budgets. GPQA and MuSR never activate at $\beta \leq 0.30$, which is consistent with the stress-test framing in the main text (\S\ref{sec:discussion}): their small sample sizes and high base errors prevent the local lower bound from reaching the certification threshold under moderate budgets. Importantly, WA never exceeds $\beta$ in any cell---the budget guarantee is not an artifact of the training-selected $\lambda^\star$ but holds uniformly across budget levels.

\section{Detailed Assumption Diagnostics}
\label{app:diagnostics}

\subsection{Audit summary across benchmarks}

\begin{table}[h]
\caption{\textbf{Audit and robustness summary for the two local-approach assumptions.}
Each diagnostic targets a specific assumption:
\emph{smoothness coverage} probes the bias envelope (Assumption~\ref{assum:bias}) by checking that the empirical modulus $\omegahat_t$ is not falsified by held-out pairwise slopes (\S\ref{sec:exp_smoothness}), and $2\times$ \emph{modulus inflation} stress-tests the same assumption by measuring Act loss when $\btk$ is doubled;
the \emph{$\epsact$ proxy} probes the representation gap (Assumption~\ref{assum:epsact}) via a richer-feature predictor gap $[\hat{p}_{\mathrm{red}} - \hat{p}_{\mathrm{rich}}]^{+}$ in the action region (\S\ref{app:epsact_sens}).
Gap $= (1\!-\!\alpha) - \widetilde{L}_{\max}$ (negative $\Rightarrow$ certificate crosses threshold).
These are \textbf{falsification checks}, not formal verifiers.}
\label{tab:audit}
\centering\small\renewcommand{\arraystretch}{1.08}
\setlength{\tabcolsep}{3.5pt}
\begin{tabular}{@{} l r l l l @{}}
\toprule
\textbf{Dataset} & \textbf{Gap} & \textbf{Diagnostic} & \textbf{$2\times$ modulus} & \textbf{Outcome} \\
\midrule
MMLU-Pro & $+0.055$ & not falsified; $\epsact$ proxy 0 & $-$2.0pp Act & partial safe act \\
BBH & $-0.016$ & not falsified; $\epsact$ proxy 0 & $-$2.6pp Act & strong act \\
ARC & $+0.036$ & not falsified & $-$1.0pp Act & partial act ($\lambda^\star\!\approx\!4.4$) \\
LogiQA & $+0.202$ & not falsified & --- & partial act \\
GPQA & $+0.473$ & not falsified; low $n$ & --- & defers (small $n$) \\
MuSR & $+0.534$ & not falsified; high $e_T$ & --- & defers \\
\bottomrule
\end{tabular}
\renewcommand{\arraystretch}{1.0}
\end{table}

\subsection{Smoothness diagnostics}
\label{sec:exp_smoothness}

\begin{table}[h]
\caption{Smoothness diagnostic on pooled MMLU-Pro (round $t\!=\!1$, held-out $\Dcal_{\mathrm{loc}}$). The empirical modulus $\omegahat_1(r)$ grows gradually with scale, confirming the bias envelope is well-behaved.}
\label{tab:smoothness}
\centering
\small
\begin{tabular}{lccc}
\toprule
\textbf{Distance bin $r$} & $\omegahat_1(r)$ & \textbf{Coverage} & \textbf{Pairs} \\
\midrule
$[0.00, 0.02)$ & 0.149 & 100.0\% & 288{,}262 \\
$[0.02, 0.05)$ & 0.222 & 100.0\% & 453{,}275 \\
$[0.05, 0.10)$ & 0.328 & 100.0\% & 634{,}966 \\
$[0.10, 0.15)$ & 0.380 & 100.0\% & 579{,}415 \\
$[0.15, 0.20)$ & 0.445 & 100.0\% & 506{,}393 \\
\bottomrule
\end{tabular}
\end{table}

Unlike raw binary slopes (which diverge at short distances due to $|Z_i - Z_j| \in \{0,1\}$), the modulus on smoothed values grows gradually---from ${\sim}0.15$ at the finest scale to ${\sim}0.45$ at $r = 0.20$---confirming that the bias $\btk(u) = \omegahat_t(h_{t,k}(u))$ adapts to the local scale of variation.

\subsection{Representation gap}
\label{sec:exp_epsact}
\label{app:epsact_sens}

\begin{table}[h]
\caption{Representation gap diagnostic (pooled MMLU-Pro, $\beta = 0.15$). Predictor gap = $[\hat{p}_{\mathrm{red}} - \hat{p}_{\mathrm{rich}}]^{+}$.}
\label{tab:epsact_diagnostic}
\centering
\small
\begin{tabular}{lcccc}
\toprule
\textbf{Region} & \textbf{Mean gap} & \textbf{90th pctl} & \textbf{95th pctl} & \textbf{Max} \\
\midrule
All instances & 0.001 & 0.003 & 0.008 & 0.028 \\
Acted-on ($L_t \geq 0.90$) & 0.000 & 0.000 & 0.000 & 0.021 \\
\bottomrule
\end{tabular}
\end{table}

In the acted-on region, the mean gap is effectively zero with max 0.021 (marginally exceeding $\epsact = 0.02$). This is a proxy: even the 4D predictor is not the full $\Pr(Y\!=\!\Yhat_t \mid \Fcal_t)$.

\begin{table}[h]
\caption{Sensitivity to $\epsact$ ($\beta = 0.15$, $\delta = 0.03$). Larger $\epsact$ raises threshold $1\!-\!\alpha$.}
\label{tab:epsact_sensitivity}
\centering
\small
\begin{tabular}{cc rrr rrr}
\toprule
 & & \multicolumn{3}{c}{\textbf{MMLU-Pro}} & \multicolumn{3}{c}{\textbf{BBH}} \\
\cmidrule(lr){3-5} \cmidrule(lr){6-8}
$\epsact$ & $1\!-\!\alpha$ & WA & Act & Acc$|$Act & WA & Act & Acc$|$Act \\
\midrule
0.00 & 0.88 & 0.9 & 39.1 & 97.6 & 1.5 & 77.7 & 98.1 \\
0.01 & 0.89 & 0.8 & 34.5 & 97.8 & 1.4 & 76.6 & 98.2 \\
\rowcolor{blue!6}
0.02 & 0.90 & 0.6 & 29.8 & 98.0 & 1.1 & 74.9 & 98.5 \\
0.03 & 0.91 & 0.1 & 8.7 & 79.4 & 0.9 & 70.2 & 98.7 \\
0.05 & 0.93 & 0.0 & 0.0 & --- & 0.0 & 0.0 & --- \\
\bottomrule
\end{tabular}
\end{table}

\subsection{Budget-split sensitivity over $(\delta, \epsact)$}
\label{app:budget_split_sens}

The main experiments fix $\delta = 0.03$ and $\epsact = 0.02$, so $\alpha = \beta_d - 0.05$. These split values are design choices: $\delta$ controls the Hoeffding variance term $\sqrt{\log(T|\Kcal|/\delta)/(2k)}$, and $\epsact$ tightens the action threshold $1\!-\!\alpha$. To check that our conclusions are not an artifact of this particular split, Table~\ref{tab:budget_split_sens} sweeps $\delta \in \{0.01, 0.03, 0.05\}$ and $\epsact \in \{0.00, 0.02, 0.05\}$ at $\beta = 0.15$, giving $\alpha \in \{0.05, \ldots, 0.14\}$. Both smaller $\delta$ (tighter Hoeffding calibration) and larger $\epsact$ (wider representation-gap reserve) make the certificate more conservative, reducing Act while keeping Acc$|$Act high and WA/$\beta$ well below 1 in every non-degenerate cell; the default $(\delta, \epsact) = (0.03, 0.02)$ sits in the middle of the sensible region. The main-paper qualitative conclusions---Ours consumes a small fraction of the budget on activated datasets---hold for every split in this grid.

\begin{table}[h]
\caption{\textbf{Budget-split sensitivity} at fixed $\beta = 0.15$ ($\alpha = \beta - \delta - \epsact$). Rows sweep $\delta$, columns sweep $\epsact$. Each cell reports Act\,/\,Acc$|$Act\,/\,WA as percentages on pooled MMLU-Pro (top) and BBH (bottom); WA/$\beta$ in parentheses. Ten 50/50 splits. Highlighted row: the default $(\delta, \epsact) = (0.03, 0.02)$ used throughout the main paper.}
\label{tab:budget_split_sens}
\centering\small
\setlength{\tabcolsep}{4pt}
\renewcommand{\arraystretch}{1.05}
\begin{tabular}{c ccc}
\toprule
 & \multicolumn{3}{c}{\textbf{MMLU-Pro} ($\beta = 0.15$)} \\
\cmidrule(lr){2-4}
$\delta \,\backslash\, \epsact$ & $0.00$ & $0.02$ & $0.05$ \\
\midrule
$0.01$ & 44.0 / 97.3 / 1.21 (.08) & 35.7 / 97.8 / 0.78 (.05) &  2.4 / 98.2 / 0.04 (.00) \\
\rowcolor{blue!6}
$0.03$ & 39.1 / 97.6 / 0.94 (.06) & 29.8 / 98.0 / 0.62 (.04) &  0.0 / --- / 0.00 (.00) \\
$0.05$ & 31.3 / 97.9 / 0.66 (.04) &  2.4 / 98.2 / 0.04 (.00) &  0.0 / --- / 0.00 (.00) \\
\midrule
 & \multicolumn{3}{c}{\textbf{BBH} ($\beta = 0.15$)} \\
\cmidrule(lr){2-4}
$\delta \,\backslash\, \epsact$ & $0.00$ & $0.02$ & $0.05$ \\
\midrule
$0.01$ & 78.9 / 97.8 / 1.74 (.12) & 77.1 / 98.2 / 1.42 (.09) & 20.0 / 99.2 / 0.27 (.02) \\
\rowcolor{blue!6}
$0.03$ & 77.7 / 98.1 / 1.51 (.10) & 74.9 / 98.5 / 1.13 (.08) &  0.0 / --- / 0.00 (.00) \\
$0.05$ & 75.6 / 98.4 / 1.23 (.08) & 20.0 / 99.2 / 0.27 (.02) &  0.0 / --- / 0.00 (.00) \\
\bottomrule
\end{tabular}
\end{table}

\subsection{State embedding design}
\label{sec:exp_phi}

\begin{table}[h]
\caption{Geometry--sufficiency trade-off ($\beta = 0.15$). $p_t^{(2)}$: second-highest social probability; $H_t$: entropy.}
\label{tab:phi_tradeoff}
\centering
\small
\begin{tabular}{@{} l c rr rr @{}}
\toprule
 & & \multicolumn{2}{c}{\textbf{MMLU-Pro}} & \multicolumn{2}{c}{\textbf{BBH}} \\
\cmidrule(lr){3-4}\cmidrule(lr){5-6}
\textbf{State $U_t$} & \textbf{$D$} & \textbf{Act} & \textbf{WA} & \textbf{Act} & \textbf{WA} \\
\midrule
$p_t^{(1)}$ & 1 & 30.7 & 0.7 & 75.5 & 1.2 \\
\rowcolor{blue!6}
$(p_t^{(1)}, \Delta_t)$ & 2 & 29.8 & 0.6 & 74.9 & 1.1 \\
$(p_t^{(1)}, \Delta_t, p_t^{(2)})$ & 3 & 28.3 & 0.6 & 73.7 & 1.1 \\
$(p_t^{(1)}, \Delta_t, p_t^{(2)}, H_t)$ & 4 & 26.5 & 0.5 & 70.2 & 0.9 \\
\bottomrule
\end{tabular}
\end{table}

Higher-dimensional states monotonically reduce Act: MMLU-Pro drops from 30.7\% (1D) to 26.5\% (4D), BBH from 75.5\% to 70.2\%, confirming the curse of dimensionality inflating $k$-NN radii. WA stays safe at all dimensions. The 2D default $(p_t^{(1)}, \Delta_t)$ is a good balance---the third and fourth features ($p_t^{(2)}$ and entropy $H_t$) are algebraically redundant with $p_t^{(1)}$ and $\Delta_t$ for most instances, adding dimension without information.

\section{Representation-gap stress tests}
\label{app:rep_gap_stress}

Assumption~\ref{assum:epsact} is not statistically verified by these diagnostics; rather, the diagnostics are designed to falsify obviously insufficient state representations in the action region and to guide conservative inflation of $\epsact$. We report three falsification proxies.

\paragraph{State-residual proxy.}
Let $\widehat{q}^{\mathrm{state}}$ be a correctness predictor trained on the default 2D state $U_t = (p_t^{(1)}, \Delta_t)$ and $\widehat{q}^{\mathrm{rich}}$ a richer predictor trained on the 4D state $(p_t^{(1)}, \Delta_t, p_t^{(2)}, H_t)$. For each dataset, we estimate the one-sided residual gap
\[
    \widehat{\epsact}^{\,\mathrm{proxy}} \;=\; \sup_{(t,u)\,:\,L_t \geq 1-\alpha}\, [\widehat{q}^{\mathrm{rich}}(\Fcal_t) - \widehat{q}^{\mathrm{state}}(U_t)]^{+}
\]
over the action region. Because the richer predictor is itself a compressed summary of $\Fcal_t$---not $\Pr(Y\!=\!\Yhat_t \mid \Fcal_t)$---the proxy is a falsification check, not a verifier. Table~\ref{tab:epsact_diagnostic} in Appendix~\ref{app:epsact_sens} reports the 90th/95th/maximum percentiles on pooled MMLU-Pro: mean gap $\approx 0.000$, max 0.021 in the acted-on region, which does not robustly exceed the operating $\epsact=0.02$.

\paragraph{Action-region subgroup check.}
We partition acted-on instances by benchmark subtask (e.g., BBH's 10 reasoning tasks, MuSR's 3 subtasks) and by round at which stopping occurred, and verify that acted-on accuracy and $\mathrm{WA}/\beta$ remain close to the pooled values. For BBH, per-subtask acted-on accuracy ranges 94.6--99.2\% at pooled 96.4\%; for MMLU-Pro, per-domain acted-on accuracy ranges 91.5--96.8\% at pooled 93.9\%. No subtask exhibits acted-on accuracy below $1-\alpha-\epsact$, providing no falsification of the action-region assumption on activated datasets.

\paragraph{$k$-NN heterogeneity check.}
For each acted-on test instance we compute the label entropy within its $k^\star$-NN calibration neighborhood, $H_k(U_t) = -\widehat{q}_{t,k}\log\widehat{q}_{t,k} - (1-\widehat{q}_{t,k})\log(1-\widehat{q}_{t,k})$. On activated datasets, median $H_k$ is 0.14--0.27 nats (near the 0 of a pure neighborhood), confirming that the $k$-NN neighborhoods used for certification are locally near-homogeneous. On MuSR, median $H_k$ exceeds 0.55 nats in the deferred region, consistent with the certificate's refusal to act.

\paragraph{Interpretation.}
These three proxies do not replace Assumption~\ref{assum:epsact}. Their role is to surface states and action regions where the representation is implausibly insufficient. Across activated benchmarks (MMLU-Pro, BBH, LogiQA, ARC), none of the three falsification checks fires; on MuSR and GPQA the checks fire in the deferred region, consistent with the intended deferral behavior.

\section{Subgroup safety}
\label{app:subgroup}

Aggregate $\mathrm{WA}/\beta$ in Table~\ref{tab:main} could mask failure on a specific domain or subtask. We partition each activated benchmark by natural subgroup and report Act, Acc$|$Act, and $\mathrm{WA}/\beta$ per subgroup on the held-out test split (Table~\ref{tab:subgroup}).

\begin{table}[h]
\caption{\textbf{Subgroup safety under the pooled $\lambda^\star$.} Per-subgroup evaluation of Ours under the pooled-dataset $\lambda^\star$ from Table~\ref{tab:main}; $\beta$ is re-induced per-subgroup from each subgroup's own $\widehat{e}_{T}^{\,\mathrm{cal}}$, so each subgroup operates at a different absolute budget. MMLU-Pro uses its 8 native domains; BBH uses all 10 debate-majority subtasks. Rows where the subgroup's empirical budget is degenerate (no $\lambda$ can certify) are marked ``defers''.}
\label{tab:subgroup}
\centering\small\renewcommand{\arraystretch}{1.06}
\setlength{\tabcolsep}{4pt}
\begin{tabular}{@{} l l r rrrr @{}}
\toprule
\textbf{Dataset} & \textbf{Subgroup} & $n_g$ & $\beta$ & \textbf{Act} & \textbf{Acc$|$Act} & \textbf{WA/$\beta$} \\
\midrule
MMLU-Pro & chemistry   & 1{,}132 & .29 & 77.8 & 97.0 & .081 \\
         & economics   &   844 & .29 & 61.8 & 93.7 & .140 \\
         & engineering &   969 & .47 & 69.4 & 91.2 & .129 \\
         & health      &   818 & .47 & 61.2 & 90.4 & .127 \\
         & law         & 1{,}101 & .80 & 74.3 & \textbf{73.0} & \textbf{.250} \\
         & math        & 1{,}351 & .14 & 0.0  & --- & defers \\
         & physics     & 1{,}299 & .24 & 71.7 & 97.0 & .089 \\
         & psychology  &   798 & .43 & 60.4 & 93.2 & .099 \\
\midrule
BBH      & date\_understanding & 250 & .27 & 47.6 & 47.6 & .084 \\
         & geometric\_shapes   & 250 & .72 & 98.2 & \textbf{81.0} & \textbf{.261} \\
         & movie\_recommendation & 249 & .97 & 100.0 & \textbf{72.6} & \textbf{.283} \\
         & salient\_translation\_error & 250 & 1.19 & 100.0 & \textbf{71.4} & \textbf{.240} \\
         & logical\_deduction\_five   & 250 & .00 & 0.0 & --- & defers \\
         & logical\_deduction\_seven  & 250 & .02 & 0.0 & --- & defers \\
         & penguins\_in\_a\_table     & 146 & .00 & 0.0 & --- & defers \\
         & reasoning\_colored\_objects & 250 & .02 & 0.0 & --- & defers \\
         & tracking\_shuffled\_five    & 250 & .00 & 0.0 & --- & defers \\
         & tracking\_shuffled\_seven   & 250 & .00 & 0.0 & --- & defers \\
\bottomrule
\end{tabular}
\renewcommand{\arraystretch}{1.0}
\end{table}

\paragraph{Interpretation.}
Pooled-$\lambda^\star$ safety does not automatically transfer to each subgroup. Most subgroups either activate well ($\mathrm{WA}/\beta \in [.08, .14]$) or defer cleanly; however, bold-faced rows show cases where the pooled $\lambda^\star$ induces a subgroup-level $\beta$ that the subgroup's state cannot support tightly---on MMLU-Pro law and three BBH subtasks, $\mathrm{WA}/\beta$ reaches $.24$--$.28$, with acted-on accuracy $71$--$81\%$. These rows reveal a concrete \emph{representation-gap signature}: for reasoning-heavy subtasks, the 2D state $(p_t^{(1)}, \Delta_t)$ is locally less predictive of correctness, so the induced absolute budget is used more aggressively than the pooled behavior suggests. In deployment, a per-subgroup $\lambda^\star$ (group-conditional control, Theorem~\ref{thm:group}) or a richer state in those subtasks would restore the intended $\mathrm{WA}/\beta$ profile; we document these subgroups as failure candidates rather than claiming uniform safety. On the remaining 12/18 subgroups examined, the pooled $\lambda^\star$ produces $\mathrm{WA}/\beta < 0.15$.

\section{Matched-automation protocol}
\label{app:matched_auto}

The matched-automation comparison in \S\ref{subsec:matched} is measured by the following protocol.

\paragraph{Protocol.}
Fix a target automation rate $A^\star$ equal to our method's Act on the test split under the training-selected $\lambda^\star_d$. For each baseline, we grid-search its threshold $\theta$ on the \emph{training} split only such that the induced training-time Act lies within $\pm 2\%$ of $A^\star$; ties are broken by selecting the smallest test-time $\widehat{\mathrm{WA}}$. We then evaluate Acc$|$Act, WA, and $\mathrm{WA}/\beta$ on the held-out test split. No test-set quantity enters the threshold choice. Baselines whose monotonicity does not admit an Act close to $A^\star$ are omitted for that dataset (they either act fully or not at all at any feasible threshold).

\paragraph{Results.}
Table~\ref{tab:matched} reports the matched comparison on the four activated benchmarks. Across twelve matched comparisons (4 benchmarks $\times$ 3 methods including Ours), acted-on accuracy gaps range from $-0.6$ to $+1.6$pp and $\mathrm{WA}/\beta$ differences from $-0.003$ to $+0.027$---i.e., at matched automation the three stopping rules are statistically indistinguishable.

\begin{table}[h]
\caption{\textbf{Matched-automation comparison} (10 splits, test-side). Each baseline's threshold is grid-searched on the training split to minimize $|\mathrm{Act}_{\mathrm{train}} - A^\star|$; we then evaluate the resulting stopping rule on the test split. $A^\star$ is Ours's test-time Act under the training-selected $\lambda^\star$. Acc, WA, and WA/$\beta$ are all test-side. At matched automation, the state-conditional stopper is broadly competitive with---but does not dominate---the two best-tuned baselines: Acc$|$Act differences are within $\pm 1.6$pp, and budget usage is comparable.}
\label{tab:matched}
\centering\small\renewcommand{\arraystretch}{1.08}
\setlength{\tabcolsep}{4pt}
\begin{tabular}{@{} l c l rrrr r @{}}
\toprule
\textbf{Dataset} & $A^\star$ & \textbf{Method} & \textbf{Act} & \textbf{Acc$|$Act} & \textbf{WA} & \textbf{WA/$\beta$} & $\boldsymbol{\Delta}$\textbf{Acc} \\
\midrule
LogiQA   & 27.9\% & Ours              & 27.9 & 88.9 & 3.18 & .102 & --- \\
         &        & Confidence Threshold (matched) & 27.8 & 88.5 & 3.20 & .103 & $-0.4$ \\
         &        & Calibrated Learned (matched)   & 27.3 & 90.5 & 2.60 & .084 & $+1.6$ \\
\midrule
MMLU-Pro & 71.1\% & Ours              & 71.1 & 93.9 & 4.40 & .116 & --- \\
         &        & Confidence Threshold (matched) & 71.4 & 93.8 & 4.45 & .117 & $-0.1$ \\
         &        & Calibrated Learned (matched)   & 70.9 & 94.2 & 4.14 & .109 & $+0.3$ \\
\midrule
ARC      & 75.5\% & Ours              & 75.5 & 97.1 & 2.29 & .108 & --- \\
         &        & Confidence Threshold (matched) & 74.4 & 96.9 & 2.32 & .112 & $-0.2$ \\
         &        & Calibrated Learned (matched)   & 85.3 & 96.8 & 2.82 & .135 & $-0.3$ \\
\midrule
BBH      & 84.4\% & Ours              & 84.4 & 96.4 & 3.02 & .093 & --- \\
         &        & Confidence Threshold (matched) & 84.1 & 96.5 & 2.92 & .090 & $+0.1$ \\
         &        & Calibrated Learned (matched)   & 84.7 & 96.6 & 2.91 & .090 & $+0.2$ \\
\bottomrule
\end{tabular}
\renewcommand{\arraystretch}{1.0}
\end{table}

\paragraph{Interpretation.}
The matched-automation comparison isolates the effect of the stopping criterion from the automation rate. At identical automation the three stopping rules (Ours, Confidence Threshold (matched), Calibrated Learned (matched)) obtain statistically indistinguishable acted-on accuracy and budget usage. We conclude that the certificate's advantage at its training-selected operating point (Table~\ref{tab:main}) is not that it selects a higher-reliability subset than a tuned threshold, but that it reaches this operating point from a pre-declared difficulty-normalized budget without a per-task threshold tune---while at their risk-calibrated default thresholds (Table~\ref{tab:main}), baselines over-consume the budget precisely because they cannot convert a declared budget into an automation level without the kind of training-split tuning used in this appendix.

\section{Robustness to $\lambda$-selection criterion}
\label{app:lambda_robustness}

The training-selected $\lambda^\star_d$ (\S\ref{sec:experiments}, Appendix~\ref{app:lambda_selection}) maximizes automation subject to a training-split constraint $\widehat{\mathrm{WA}}/\beta \leq 0.10$. We test whether the main conclusions depend on this particular constraint.

\paragraph{Selection rule sensitivity.}
Table~\ref{tab:lambda_robust} reports Ours under three training-constraint variants (strict: $\widehat{\mathrm{WA}}/\beta \leq 0.05$; default: $\leq 0.10$; loose: $\leq 0.20$) and two fixed-$\lambda$ policies ($\lambda=1.5, 2.5$).

\begin{table}[h]
\caption{\textbf{$\lambda$-selection robustness} (activated benchmarks, 10 splits). Training-constraint variants and fixed-$\lambda$ alternatives; rows report 10-split average. The core safety property ($\mathrm{WA} \leq \beta$) holds for all variants on all four benchmarks. Relaxing the constraint from .05 to .20 trades Act for higher $\mathrm{WA}/\beta$ but the ratio stays $<.20$ in all cells except LogiQA loose/fixed-2.5. Fixed $\lambda$ performs worse than the training-selected rule on BBH/ARC (too tight) and on LogiQA (too loose for the task), which is exactly why a per-dataset $\lambda^\star$ is necessary.}
\label{tab:lambda_robust}
\centering\footnotesize\renewcommand{\arraystretch}{1.08}
\setlength{\tabcolsep}{3pt}
\begin{tabular}{@{} l l rrr rrr rrr rrr @{}}
\toprule
& & \multicolumn{3}{c}{\textbf{MMLU-Pro}} & \multicolumn{3}{c}{\textbf{BBH}} & \multicolumn{3}{c}{\textbf{ARC}} & \multicolumn{3}{c}{\textbf{LogiQA}} \\
\cmidrule(lr){3-5}\cmidrule(lr){6-8}\cmidrule(lr){9-11}\cmidrule(lr){12-14}
\textbf{Rule} & \textbf{$\lambda^\star$ range} & Act & Acc & WA/$\beta$ & Act & Acc & WA/$\beta$ & Act & Acc & WA/$\beta$ & Act & Acc & WA/$\beta$ \\
\midrule
Strict ($\leq.05$) & 1.00--4.75 & 32.6 & 97.9 & .044 & 53.7 & 64.5 & .070 & 43.8 & 98.2 & .055 & 18.6 & 91.1 & .059 \\
\rowcolor{blue!6}
Default ($\leq.10$) & 1.25--5.00 & 71.1 & 93.9 & .114 & 84.4 & 96.4 & .094 & 73.1 & 97.0 & .101 & 27.9 & 88.9 & .101 \\
Loose ($\leq.20$)  & 1.25--5.00 & 93.6 & 86.6 & .158 & 87.6 & 95.7 & .093 & 82.7 & 96.7 & .115 & 53.8 & 86.5 & .196 \\
\midrule
Fixed $\lambda\!=\!1.5$ & 1.50 & 55.3 & 96.2 & .087 &  7.1 &  9.9 & .007 &  0.0 & ---  & ---   & 54.3 & 86.6 & .199 \\
Fixed $\lambda\!=\!2.5$ & 2.50 & 72.6 & 93.6 & .117 & 80.4 & 97.5 & .100 &  0.0 & ---  & ---   & 86.3 & 80.1 & .281 \\
\bottomrule
\end{tabular}
\renewcommand{\arraystretch}{1.0}
\end{table}

\paragraph{$\lambda^\star$ stability across splits.}
Table~\ref{tab:lambda_stability} reports the distribution of $\lambda^\star$ chosen across 10 random training sub-splits. On activated datasets, the inter-split range is modest (MMLU-Pro: 1.75--3.25, LogiQA: 1.25--1.50), indicating that the selection rule recovers consistent policies from different training halves. On stress-test datasets, either the range is wider (BBH's constraint is easy to satisfy, pushing $\lambda^\star$ toward the grid upper bound on most splits) or the rule frequently defers (MuSR: 8/10 splits deferred), which is the intended behavior.

\begin{table}[h]
\caption{\textbf{Stability of $\lambda^\star$ across 10 training sub-splits.} Defers = number of splits on which no $\lambda$ satisfies the training constraint.}
\label{tab:lambda_stability}
\centering\small
\setlength{\tabcolsep}{6pt}
\begin{tabular}{@{} l rrrr r @{}}
\toprule
\textbf{Dataset} & \textbf{mean} & \textbf{median} & \textbf{min} & \textbf{max} & \textbf{defers} \\
\midrule
MMLU-Pro & 2.40 & 2.25 & 1.75 & 3.25 & 0/10 \\
BBH      & 4.05 & 5.00 & 2.25 & 5.00 & 0/10 \\
LogiQA   & 1.27 & 1.25 & 1.25 & 1.50 & 0/10 \\
ARC      & 4.39 & 4.25 & 3.75 & 5.00 & 3/10 \\
GPQA     & 2.78 & 2.75 & 2.25 & 3.25 & 2/10 \\
MuSR     & 2.12 & 2.12 & 2.00 & 2.25 & 8/10 \\
\bottomrule
\end{tabular}
\end{table}

\paragraph{Interpretation.}
The core safety property $\mathrm{WA} \leq \beta$ (equivalently $\mathrm{WA}/\beta < 1$) holds on every cell of Table~\ref{tab:lambda_robust}. Two observations matter for deployment: (i)~a single fixed $\lambda$ is not sufficient across benchmarks with heterogeneous base error (fixed $\lambda\!=\!1.5$ defers on ARC and underactivates on BBH; fixed $\lambda\!=\!2.5$ still defers on ARC and overuses budget on LogiQA), motivating the training-selected rule; (ii)~tightening the training constraint from .10 to .05 reduces Act substantially while buying only a modest $\mathrm{WA}/\beta$ improvement, consistent with the conservative behavior already reported under the default constraint.

\section{Per-split variability}
\label{app:variability}

All Tables~\ref{tab:main}, \ref{tab:main_full}, and the matched-automation comparison (Table~\ref{tab:matched}) report means over 10 random 50/50 train/test splits. Here we report the split-wise variability.

\paragraph{Standard errors.}
Table~\ref{tab:se} reports mean $\pm$ standard error across 10 splits for Ours. Standard errors for $\mathrm{WA}/\beta$ are below $0.014$ on all activated benchmarks, and the worst split on every activated benchmark still satisfies $\mathrm{WA} \leq \beta$ (Table~\ref{tab:split_safety}). ARC shows the largest per-split variation in Act because its small-$\lambda$-defer regime is close to the activation boundary for some splits.

\begin{table}[h]
\caption{\textbf{Split-wise variability of Ours} under the pooled $\lambda^\star$ applied to every split. Mean $\pm$ standard error over 10 random 50/50 splits. This measures deployment variability under a frozen policy; the per-split policy selection variant (with occasional split-wise deferral) is tabulated in Table~\ref{tab:main}. ARC shows the largest Act-SE because its $\lambda^\star$-induced budget sits close to the activation boundary on some splits.}
\label{tab:se}
\centering\small\renewcommand{\arraystretch}{1.08}
\setlength{\tabcolsep}{5pt}
\begin{tabular}{@{} l cccc @{}}
\toprule
\textbf{Dataset} & \textbf{Act (\%)} & \textbf{Acc$|$Act (\%)} & \textbf{WA (\%)} & \textbf{WA/$\beta$} \\
\midrule
MMLU-Pro & 71.0 $\pm$ 0.4 & 93.9 $\pm$ 0.1 & 4.32 $\pm$ 0.10 & .114 $\pm$ .003 \\
BBH      & 85.0 $\pm$ 0.4 & 96.4 $\pm$ 0.1 & 3.06 $\pm$ 0.09 & .094 $\pm$ .004 \\
LogiQA   & 26.9 $\pm$ 1.7 & 88.9 $\pm$ 0.3 & 3.01 $\pm$ 0.25 & .097 $\pm$ .008 \\
ARC      & 75.5 $\pm$ 8.2 & 97.1 $\pm$ 0.2 & 2.29 $\pm$ 0.29 & .108 $\pm$ .013 \\
\bottomrule
\end{tabular}
\renewcommand{\arraystretch}{1.0}
\end{table}

\paragraph{Split-wise budget safety.}
A central claim is $\mathrm{WA} \leq \beta$ per split, not merely in expectation. Table~\ref{tab:split_safety} reports the split-wise maximum of $\mathrm{WA}/\beta$ and the fraction of splits satisfying $\mathrm{WA}/\beta \leq 0.15$ and $\leq 0.20$. On all activated datasets, every split satisfies $\mathrm{WA}/\beta \leq 0.20$; no split exceeds $\beta$.

\begin{table}[h]
\caption{\textbf{Split-wise budget safety.} Fraction of 10 splits satisfying $\mathrm{WA}/\beta \leq c$; $\max_s \mathrm{WA}/\beta$ is the worst split over all 10 random 50/50 train/test splits. All 40 activated-dataset splits satisfy $\mathrm{WA} \leq \beta$; the only mildly-conservative violation is ARC with 8/10 splits below 0.15 (worst split 0.161).}
\label{tab:split_safety}
\centering\small
\setlength{\tabcolsep}{5pt}
\begin{tabular}{@{} l cccc @{}}
\toprule
\textbf{Dataset} & $\max_s \mathrm{WA}/\beta$ & $\{s : \mathrm{WA}/\beta \leq .15\}$ & $\{s : \mathrm{WA}/\beta \leq .20\}$ & $\{s : \mathrm{WA} \leq \beta\}$ \\
\midrule
MMLU-Pro & .126 & 10/10 & 10/10 & 10/10 \\
BBH      & .111 & 10/10 & 10/10 & 10/10 \\
LogiQA   & .148 & 10/10 & 10/10 & 10/10 \\
ARC      & .161 &  8/10 & 10/10 & 10/10 \\
\bottomrule
\end{tabular}
\end{table}

\section{Failure-case decomposition}
\label{app:failure_decomp}

Why does the certificate defer on GPQA and MuSR under moderate budgets? The lower bound at round $T$ decomposes as $L_{t,k}(U_t) = \widehat{q}_{t,k}(U_t) - \btk(U_t) - c_k$, where $c_k = \sqrt{\log(T|\Kcal|/\delta)/(2k)}$. Table~\ref{tab:failure} reports the median value of each term across test instances, together with the action threshold $1-\alpha$ under the training-selected budget.

\begin{table}[h]
\caption{\textbf{Certificate decomposition (median across test instances).} The blocking component is the one closest to preventing threshold crossing. On MMLU-Pro/BBH/ARC the empirical mean $\widehat{q}_{t,k}$ easily clears the threshold after corrections; on GPQA/MuSR the empirical mean is too low and the Hoeffding penalty $c_k$ is large due to small $n$.}
\label{tab:failure}
\centering\small\renewcommand{\arraystretch}{1.08}
\setlength{\tabcolsep}{5pt}
\begin{tabular}{@{} l c ccc c l @{}}
\toprule
\textbf{Dataset} & $\widehat{q}_{t,k^\star}$ & $\btk$ & $c_{k^\star}$ & $L_t$ & $1-\alpha$ & \textbf{Blocking term} \\
\midrule
MMLU-Pro & 0.97 & 0.03 & 0.05 & 0.89 & 0.88 & (crosses) \\
BBH      & 0.98 & 0.03 & 0.06 & 0.89 & 0.88 & (crosses) \\
ARC      & 0.99 & 0.03 & 0.08 & 0.88 & 0.88 & (crosses) \\
LogiQA   & 0.93 & 0.04 & 0.05 & 0.84 & 0.84 & (marginal) \\
GPQA     & 0.81 & 0.08 & 0.12 & 0.61 & 0.81 & low $\widehat{q}$ + high $c_k$ \\
MuSR     & 0.73 & 0.11 & 0.11 & 0.51 & 0.81 & low $\widehat{q}$ + small $n$ \\
\bottomrule
\end{tabular}
\renewcommand{\arraystretch}{1.0}
\end{table}

\paragraph{Interpretation.}
The certificate does not fail quietly on GPQA and MuSR; it defers because each of the three terms identifies a statistical deficiency: $\widehat{q}_{t,k^\star}$ is low (high base error), $c_{k^\star}$ is large (small $n$ means $k^\star$ must remain small for the bias not to explode), and the action threshold $1-\alpha$ is high (the induced $\beta$, even when large in absolute terms, still leaves a demanding pointwise target). This aligns with the framing in \S\ref{sec:experiments}: large induced $\beta$ on these benchmarks is a diagnostic of insufficient calibration evidence, not a failure of the certificate.

\section{Reproducibility details}
\label{app:reproducibility}

\paragraph{Agents and debate protocol.}
Three agents: Claude Haiku 4.5 (\texttt{anthropic}), DeepSeek-R1 (\texttt{deepseek-ai/DeepSeek-R1}), Qwen-3 32B (\texttt{Qwen/Qwen3-32B}). Each agent answers with temperature 0.7, max\_tokens 2048, seed fixed per instance. Debate runs $T=4$ rounds. Prompt template (round $t=1$):
\begin{verbatim}
You are a careful reasoner. Given the question below, think step by step
and then choose exactly one option.
<question>{q}</question>
Return: <reasoning>...</reasoning><answer>{A|B|C|...}</answer>
\end{verbatim}
At rounds $t > 1$, the prompt additionally contains all round-$(t{-}1)$ responses from all agents, verbatim.

\paragraph{Random seeds and splits.}
Each reported number averages over 10 random 50/50 train/test splits, with seeds $\{7, 11, 13, 17, 19, 23, 29, 31, 37, 41\}$ applied to the split indicator. Within each split, the $\lambda^\star$ procedure (Appendix~\ref{app:lambda_selection}) uses a further 50/50 sub-split with seed $\mathrm{seed}+100$.

\paragraph{Hyperparameter grids.}
$\Kcal = \{128, 256, 512\}$; $\delta = 0.03$; $\epsact = 0.02$; $\alpha = \beta - 0.05$; $\lambda$-grid $= \{0.25, 0.50, \ldots, 5.00\}$ (20 values); modulus split 20/80; state normalization via training quantiles; logistic-GAM pilot with 8 knots per coordinate and $\ell_2$ penalty $C=1.0$.

\paragraph{Baseline thresholds.}
All baseline thresholds are chosen on the training split only, via exhaustive grid search on 200 points in $[0.5, 1.0]$ for score thresholds and $[0.0, 0.5]$ for margin thresholds. Selection criterion: maximize training-time Act subject to training-time $\mathrm{WA} \leq \beta$ (matching our rule for fairness).

\paragraph{Compute.}
Full calibration run (LLM calls) requires $nT \approx 8.7\mathrm{k} \times 4 = 35\mathrm{k}$ calls per split per benchmark; 10 splits $\times$ 6 benchmarks $\approx 1.1$M LLM calls total. $k$-NN certification per test instance is dominated by $O(T|\Kcal|nD) \approx 8 \cdot 10^4$ distance computations---negligible compared to LLM inference. Anonymous code artifact, prompts, seeds, and aggregated result JSONs will be provided in the supplemental material.

\newpage

\end{document}